\documentclass{article}

\usepackage{microtype}
\usepackage{graphicx}
\usepackage{subfigure}
\usepackage{booktabs} 
\usepackage[table]{xcolor}
\usepackage{tabu}
\usepackage{tikz}
\usetikzlibrary{bayesnet}

\usepackage{amsmath,amssymb,amsfonts,amsthm}
\newtheorem{theorem}{Theorem}
\newtheorem{lemma}{Lemma}[section]

\theoremstyle{definition}
\newtheorem{definition}{Definition}
\newtheorem{example}{Example}[section]
\usepackage{comment}

\usepackage{bm}
\usepackage[noend,algo2e]{algorithm2e}
\usepackage{adjustbox}
\newsavebox \mybox

\newcommand \from \leftarrow

\newcommand \dom {{\operatorname*{dom}}}
\newcommand \unroll {{\operatorname*{unroll}}}
\renewcommand \cdots {{\cdot\!\cdot\!\cdot}}

\newcommand{\by}{{\bm{y}}}
\newcommand{\bx}{{\bm{x}}}
\newcommand{\bw}{{\bm{w}}}
\usepackage{enumitem}

\newcommand*{\Scale}[2][4]{\scalebox{#1}{\ensuremath{#2}}}

\usepackage{colortbl}
\usepackage{pgfplots}
\usepackage{pgfplotstable}
\usepackage{tikzsymbols}

\pgfplotstableset{
    /color cells/min/.initial=0,
    /color cells/max/.initial=1000,
    /color cells/textcolor/.initial=,
    color cells/.code={%
        \pgfqkeys{/color cells}{#1}%
        \pgfkeysalso{%
            postproc cell content/.code={%
                \begingroup
                \pgfkeysgetvalue{/pgfplots/table/@preprocessed cell content}\value
                \ifx\value\empty
                    \endgroup
                \else
                \pgfmathfloatparsenumber{\value}%
                \pgfmathfloattofixed{\pgfmathresult}%
                \let\value=\pgfmathresult
                \pgfplotscolormapaccess
                    [\pgfkeysvalueof{/color cells/min}:\pgfkeysvalueof{/color cells/max}]
                    {\value}
                    {\pgfkeysvalueof{/pgfplots/colormap name}}%
                \pgfkeysgetvalue{/pgfplots/table/@cell content}\typesetvalue
                \pgfkeysgetvalue{/color cells/textcolor}\textcolorvalue
                \toks0=\expandafter{\typesetvalue}%
                \xdef\temp{%
                    \noexpand\pgfkeysalso{%
                        @cell content={%
                            \noexpand\cellcolor[rgb]{\pgfmathresult}%
                            \noexpand\definecolor{mapped color}{rgb}{\pgfmathresult}%
                            \ifx\textcolorvalue\empty
                            \else
                                \noexpand\color{\textcolorvalue}%
                            \fi
                            \the\toks0 %
                        }%
                    }%
                }%
                \endgroup
                \temp
                \fi
            }%
        }%
    }
}

\usepackage{fancyvrb}
\usepackage{listings}
\usepackage{courier}
\usepackage[accepted,nohyperref]{icml2019}

\icmltitlerunning{Tensor Variable Elimination for Plated Factor Graphs}

\begin{document}

\twocolumn[
\icmltitle{Tensor Variable Elimination for Plated Factor Graphs}



\icmlsetsymbol{equal}{*}

\begin{icmlauthorlist}
\icmlauthor{Fritz Obermeyer}{equal,uber}
\icmlauthor{Eli Bingham}{equal,uber}
\icmlauthor{Martin Jankowiak}{equal,uber}
\icmlauthor{Justin Chiu}{harvard}
\icmlauthor{Neeraj Pradhan}{uber}
\icmlauthor{Alexander M. Rush}{harvard}
\icmlauthor{Noah Goodman}{uber,stanford}
\end{icmlauthorlist}

\icmlaffiliation{uber}{Uber AI Labs}
\icmlaffiliation{stanford}{Stanford University}
\icmlaffiliation{harvard}{Harvard University}

\icmlcorrespondingauthor{Fritz Obermeyer}{fritzo@uber.com}

\icmlkeywords{Tensors, Message Passing, Probabilistic Programming}

\vskip 0.3in
] 



\printAffiliationsAndNotice{\icmlEqualContribution} 

\tikzstyle{factor} = [rectangle, draw]
\tikzstyle{left plate caption} = [caption, node distance=0, inner sep=0pt,
below left=0pt and 0pt of #1.south west]
\tikzstyle{right plate caption} = [caption, node distance=0, inner sep=0pt,
below left=0pt and 0pt of #1.south east]

\begin{abstract}
A wide class of machine learning algorithms can be reduced to variable elimination on factor graphs.
While factor graphs provide a unifying notation for these algorithms, they do not provide a compact way to express repeated structure when compared to plate diagrams for directed graphical models.
To exploit efficient tensor algebra in graphs with plates of variables, we generalize undirected factor graphs to plated factor graphs and variable elimination to a tensor variable elimination algorithm that operates directly on plated factor graphs.
Moreover, we generalize complexity bounds based on treewidth and characterize the class of plated factor graphs for which inference is tractable.
As an application, we integrate tensor variable elimination into the Pyro probabilistic programming language to enable exact inference in discrete latent variable models with repeated structure.
We validate our methods with experiments on both directed and undirected graphical models, including applications to polyphonic music modeling, animal movement modeling, and latent sentiment analysis.
\end{abstract}

\section{Introduction} \label{sec:intro}

Factor graphs \cite{kschischang2001factor} provide a unifying
representation for a wide class of machine learning algorithms as
undirected bipartite graphs between variables and factors.  Factor
graphs can be used with both directed and undirected graphical models
to represent probabilistic inference algorithms performed by variable
elimination \cite{pearl1986fusion, lauritzen1988local}. In the most
common case, variable elimination is performed by sum-product
inference, but other variable elimination algorithms can be derived
through alternative semirings and adjoints.

In recent years researchers have exploited the equivalence of
sum-product on discrete factor graphs and tensor contraction
\cite{hirata2003tensor,smith2018opt_einsum}. Standard tensor
contraction can provide efficient implementations of sum-product
inference and the tensor contraction operations can be generalized to
alternate semirings
\cite{kohlas2008semiring,belle2016semiring,khamis2016faq}.

Yet, a major downside of factor graphs as an intermediary
representation is that they discard useful information from
higher-level representations, in particular, \textit{repeated structure}.
Directed graphical models explicitly denote repeated structure through
plate notation \cite{buntine1994operations}. Plates have not seen
widespread use in factor graphs or their inference algorithms (\cite{dietz2010directed} being an exception for directed factor graphs).
Nor have plates been exploited by tensor contraction, despite the
highly parallel nature of variable elimination algorithms.
This gap can result in suboptimal algorithms, since 
repeated structure can provide information that can be directly exploited for inference optimizations.

In this work we consider the class of \emph{plated factor graphs} and the
corresponding tensor variable elimination algorithms. We propose a
natural definition for plated factor graphs and lift a
number of classic results and algorithms from factor graphs to the
plated setting. In particular, we generalize treewidth-based bounds
on computational complexity for factor graphs to bounds depending on
plate sizes, and characterize the boundary between plated factor graphs
leading to computational complexity either polynomial or exponential
in the sizes of plates in the factor graph.

We consider several different applications of these techniques.
First we describe how plated factor graphs can provide an efficient
intermediate representation for generative models with discrete
variables, and incorporate a tensor-based contraction into the
\textit{Pyro} probabilistic programming language. Next, we develop
models for three real-world problems: polyphonic music modeling,
animal movement modeling, and latent sentiment analysis. These models combine directed networks (Bayesian networks), undirected
networks (conditional random fields), and deep neural networks. We
show how plated factor graphs can be used to concisely
represent structure and provide efficient general-purpose
inference through tensor variable elimination.

\section{Related Work} \label{sec:related}

Dense sum-product problems have been studied by the HPC community under the name tensor contraction.
\citet{hirata2003tensor} implemented a Tensor Contraction Engine for performing optimized parallel tensor contractions, an instance of dense sum-product message passing.
\citet{solomonik2014massively} implement a similar framework for large-scale distributed computation of tensor contractions.
\citet{wiebe2011einsum} implemented a popular interface \lstinline$np.einsum$ for performing tensor contractions in NumPy.
\citet{smith2018opt_einsum} implement a divide-and-conquer optimizer for \lstinline$einsum$ operations; we extended their implementation and use it as a primitive for non-plated variable elimination throughout this paper. \citet{abseher2017improving} also address the problem of finding efficient junction tree decompositions.

Sparse sum-product problems have a long history in database query optimization, as they can be seen as a specific type of a join-groupby-aggregate query.
\citet{kohlas2008semiring} define an abstract framework for variable elimination over valuations and semirings.
\citet{khamis2016faq} formulate an abstract variable elimination algorithm, prove complexity bounds, and connect these algorithms to database query optimization. Lifted inference algorithms \cite{taghipour2013lifted} developed for probabilistic databases are also concerned with extending classical methods for sum-product problems to exploit repeated structure or symmetry in graphical models.



In the context of probabilistic inference,
\citet{bilmes2010dynamic} leverages repeated (typically dynamic) structure in graphical models to quickly compute a variable elimination schedule that is then executed sequentially.
By contrast our algorithm addresses the narrower class of models which exclude dependencies between plate instances, and can thereby compute a schedule for parallel variable elimination.
Infer.Net \citep{InferNET18} introduces a \lstinline{ForEach} construct that enables parallelization; our \lstinline{pyro.plate} construct similarly enables parallism but also declares statistical independence.

\section{Model: Plated Factor Graphs} \label{sec:graphs}

\begin{definition}
A \emph{factor graph} is a bipartite graph $(V,F,E)$ whose vertices are either variables $v\in V$ or factors $f\in F$, and whose edges $E\subseteq V\times F$ are pairs of vertices. We say factor $f$ involves variable $v$ iff $(v,f)\in E$.
Each variable $v$ has domain $\operatorname{dom}(v)$, and each factor $f$ involving variables $\{v_1,\dots,v_K\}$ maps values $\underline x\in\operatorname{dom}(v_1)\times\cdots\times \operatorname{dom}(v_K)$ to scalars.
\end{definition}

In this work we are interested in discrete factor graphs where variable domains are finite and factors $f$ are tensors with one dimension per neighboring variable.
The key quantity of interest is defined through the sum-product contraction,
\begin{align*}
&\textsc{SumProduct}(F, \{v_1,\dots,v_K\}) = \\
&\; \sum_{x_1\in\dom(v_1)}\!\cdots\!\sum_{x_K\in\dom(v_K)}
  \prod_{f\in F} f[v_1\!=\!x_1,\dots,v_K\!=\!x_K]
\end{align*}
where we use named tensor dimensions and assume that tensors broadcast by ignoring uninvolved variables, i.e. if $(v,f)\notin E$ then $f[v=x,v'=x']=f[v'=x']$.


\begin{definition}
A \emph{plated factor graph} is a labeled bipartite graph $(V,F,E,P)$ whose vertices are labeled by the \emph{plates} on which they are replicated $P:V\cup F\to \mathcal P(B)$, where $B$ is a set of plates. We require that each factor is in each of the plates of its variables:
${\forall (v,f)\in E,\; P(v)\subseteq P(f)}$.\footnote{While the final requirement $P(f)\supseteq P(v)$ could perhaps be relaxed,  it is required if factors $f$ are to be represented as multi-dimensional tensors.}
\end{definition}

To instantiate a plated factor graph, we assume a map $M$ that specifies the number of times $M(b)$ to replicate each plate $b$.
Under this definition each plated factor is represented by a tensor with dimensions for both its plates and its involved variables. Using the same partial tensor notation above, we can access a specific grounded factor by indexing its plate dimensions, i.e.~$f[b_1 =i_1, \ldots, b_L=i_L]$, where for each dimension $ i_l \in \{1,\ldots, M(b_l)\}$ and $f[b= i, b'=i'] = f[b=i]$ if $b' \not \in P(f)$.   

The key operation for plated factor graphs will be ``unrolling'' to 
standard factor graphs.\footnote{This operation is conceptual: it is used in our definitions and proofs, but none of our algorithms explicitly unroll factor graphs.}
First define the following plate notation for either a factor or variable $z$: ${\cal M}_z(b) =  \{1, \ldots, M(b)\}$ if $ b \in P(z)$ and $\{1\}$ otherwise. 
This is  the set of indices that index into the replicated variable or factor.
Now define a function to unroll a plate, 
\[(V', F', E', P') = \unroll((V,F,E,P), M, b)\]
where $v_i$ indicates an unrolled index of $v$ and,
\begin{align*}
V' &=   \{v_i \ | \  v \in V,\  i \in {\cal M}_v(b) \}   \\
F' &= \{f_i \  |\  f \in F, \ i \in {\cal M}_f(b) \}  \\
E' &= \{ (v_i, f_j)\  |\ (v, f) \in E,\, i \in {\cal M}_v(b),\, j \in {\cal M}_f(b),  \\
&\hspace{2cm} \ (i = j) \lor b \not \in (P(v)\cap P(f))  \}  \\ 
P'(z) &= P(z) \setminus \{b\} 
\end{align*}

Sum-product contraction naturally generalizes to plated factor graphs via unrolling. We define:
\begin{align*}
\textsc{PlatedSumProduct}(G, M) \equiv \textsc{SumProduct}(F', V')
\end{align*}
where $F'$ and $V'$ are constructed by unrolling each plate $b$ in the 
original plated factor graph $G$.

\begin{example} \label{ex:tractable}
Consider a plated factor graph with two variables $X,Y$, three factors $F,G,H$, and two nested plates:

\begin{center}
\begin{tikzpicture}
\node[factor](F) {$F$};
\node[latent, xshift=1.25cm](X) {$X$};
\node[factor, xshift=2.5cm](H) {$H$};
\node[latent, xshift=3.75cm](Y) {$Y$};
\node[factor, xshift=5cm](G) {$G$};
\plate {a}{(H)}{$J$};
\plate [inner sep= 0.3cm] {b}{(H)(Y)(G)}{};
\node [right plate caption=b-wrap, yshift=-0.2cm, xshift=0.2cm] {$I$};

\draw (F) -- (X);
\draw (X) -- (H);
\draw (H) -- (Y);
\draw (Y) -- (G);
\end{tikzpicture}
\end{center}

Assuming sizes $I=2$ and $J=3$, this plated factor graph unrolls to a factor graph:
\begin{center}
\resizebox {.55\columnwidth} {!} {
\begin{tikzpicture}

\node[factor](F) {$F$};
\node[latent, xshift=1.25cm](X) {$X$};
\node[factor, xshift=2.5cm, yshift=1.5cm](H11) {$H_{11}$};
\node[factor, xshift=2.5cm, yshift=0.9cm](H12) {$H_{12}$};
\node[factor, xshift=2.5cm, yshift=0.3cm](H13) {$H_{13}$};

\node[factor, xshift=2.5cm, yshift=-0.3cm](H21) {$H_{21}$};
\node[factor, xshift=2.5cm, yshift=-0.9cm](H22) {$H_{22}$};
\node[factor, xshift=2.5cm, yshift=-1.5cm](H23) {$H_{23}$};

\node[latent, xshift=3.75cm, yshift=0.9cm](Y1) {$Y_1$};
\node[factor, xshift=5cm, yshift=0.9cm](G1) {$G_1$};
\node[latent, xshift=3.75cm, yshift=-0.9cm](Y2) {$Y_2$};
\node[factor, xshift=5cm, yshift=-0.9cm](G2) {$G_2$};

\draw (F) -- (X);
\draw[-] (X) to [bend left=20] (H11);
\draw[-] (X) to [bend left=10] (H12);
\draw (X) -- (H13);
\draw (H11) -- (Y1);
\draw (H12) -- (Y1);
\draw (H13) -- (Y1);
\draw (Y1) -- (G1);

\draw (X) -- (H21);
\draw[-] (X) to [bend right=10] (H22);
\draw[-] (X) to [bend right=20] (H23);
\draw (H21) -- (Y2);
\draw (H22) -- (Y2);
\draw (H23) -- (Y2);
\draw (Y2) -- (G2);
\end{tikzpicture}
} 
\end{center}
\end{example}

\begin{example} \label{ex:intractable}
Consider a plated factor graph with two variables, one factor, and two overlapping non-nested plates, denoting a Restricted Boltzmann Machine (RBM) \cite{smolensky1986information}:
\begin{center}
\begin{tikzpicture}
\node[latent](X) {$X$};
\node[factor, xshift=1.25cm](F) {$F$};
\node[latent, xshift=2.5cm](Y) {$Y$};
\plate [inner sep= 0.3cm, yshift=-0.05cm] {b}{(X)(F)}{};
\node [left plate caption=b-wrap, yshift=-0.2cm]{$I$};
\plate [inner sep= 0.3cm, yshift=0.2cm] {b}{(F)(Y)}{};
\node [right plate caption=b-wrap, yshift=1.1cm, xshift=0.2cm] {$J$};

\draw (X) -- (F);
\draw (F) -- (Y);
\end{tikzpicture}

\end{center}
Assuming sizes $I=2$ and $J=2$, this plated factor graph unrolls to a factor graph:
%



%
\begin{center}
\resizebox {.55\columnwidth} {!} {
\begin{tikzpicture}
\node[latent, xshift=1.25cm, yshift=1cm] (X1) {$X_1$};
\node[latent, xshift=1.25cm, yshift=0cm] (X2) {$X_2$};

\node[factor, xshift=2.5cm, yshift=1.2cm] (F11) {$F_{11}$};
\node[factor, xshift=3.75cm, yshift=1cm] (F12) {$F_{12}$};
\node[factor, xshift=2.75cm, yshift=0cm] (F21) {$F_{21}$};
\node[factor, xshift=4.0cm, yshift=-.2cm] (F22) {$F_{22}$};

\node[latent, xshift=5.25cm, yshift=1cm] (Y1) {$Y_1$};
\node[latent, xshift=5.25cm, yshift=0cm] (Y2) {$Y_2$};

\edge[-] {X1} {F11};
\draw[-] (X1) to [bend right=20] (F12);
\edge[-] {X2} {F21};
\draw[-] (X2) to [bend right=20] (F22);
\edge[-] {F22} {Y2};
\draw[-] (F12) to [bend left=10] (Y2);
\draw[-] (F11) to [bend left=20] (Y1);
\draw[-] (F21) to [bend right=10] (Y1);
\end{tikzpicture}
} 
\end{center}
\end{example}

\section{Inference: Tensor Variable Elimination} \label{sec:algorithm+complexity}

We now describe an algorithm---tensor variable elimination---for computing \textsc{PlatedSumProduct} on a tractable subset of plated factor graphs. 
We show that variable elimination cannot be generalized to run in polynomial time on all plated factor graphs and that 
the algorithm succeeds for exactly those plated factor graphs that can be run in polynomial time.
Finally, we briefly discuss extensions of the algorithm, including a plated analog of the Viterbi algorithm on factor graphs.


\subsection{An algorithm for tensor variable elimination} \label{sec:algorithm}

The main algorithm is formulated in terms of several standard functions:
$\textsc{SumProduct}(F,V)$, introduced above, computes sum-product contraction of a set of tensors $F$ along a subset $V$ of their dimensions (here always variable dimensions) via variable elimination%
\footnote{Available in many machine learning libraries as \lstinline$einsum$.};
$\textsc{Product}(f,\underline b,M)$ product-reduces a single tensor $f$ along a subset $\underline b$ of its dimensions (here always plate dimensions) 
\begin{align*}
&\textsc{Product}(f, \{b_1,\dots, b_L\}, M) = \\
&\hspace{1.5cm}
  \prod_{i_1=1}^{M(b_1)}\cdots\prod_{i_L=1}^{M(b_L)}
  f[b_1\!=\!i_1,\dots,b_L\!=\!i_L]
\end{align*}
and $\textsc{Partition}(V,F,E)$ separates a bipartite graph into its connected components.
\textsc{SumProduct} and \textsc{Product} each return a tensor, with dimensions corresponding to remaining plates and variables, unless all dimensions get reduced, in which case they return a scalar. Intuitively, \textsc{SumProduct} eliminates variables and \textsc{Product} eliminates plates.

At a high level, the strategy of the algorithm is to greedily eliminate variables and plates along a tree of factors.
At each stage it picks the most deeply nested plate set, which we call a leaf plate.
It eliminates all variables in exactly that plate set via standard variable elimination, producing a set of reduced factors.
Each reduced factor is then replaced by a product factor that eliminates one or more plates.
Repeating this procedure until no variables or plates remain, all scalar factors are finally combined with a product operation.

Algorithm~\ref{alg:plated-sumproduct} specifies the full algorithm. Elimination of variables and plates proceeds
by modifying the input plated factor graph $(V,F,E,P)$ and a partial result $S$ containing scalars.
Both loops preserve the invariant that $(V,F,E,P)$ is a valid plated factor graph and preserve the quantity
\begin{align*}
  &\textsc{PlatedSumProduct}((V,F,E,P), M) \\
  &\hspace{1cm} \times \textsc{SumProduct}(S,\{\}).
\end{align*}
At each leaf plate set $L\subseteq B$, the algorithm decomposes that plate set's factor graph into connected components.
Each connected component is \textsc{SumProduct}-contracted to a single factor $f$ with no variables remaining in the leaf plate set $L$.
If the resulting factor $f$ has no more variables, it is \textsc{Product}-reduced to a single scalar.
Otherwise the algorithm seeks a plate set $L'\subsetneq L$ where other variables of $f$ can be eliminated; $f$ is partially \textsc{Product}-reduced to a factor $f'$ that is added back to the plated factor graph, including edges from $V_f\times F_c$ that had been removed. Finally, when no more variables or plates remain, all scalar factors are product-combined by a trivial ${\textsc{SumProduct}(S,\{\})}$.
The algorithm can fail with \textbf{error} if the search for a next plate set $L'$ fails.

\newcommand \tab {{\phantom o}}
\begin{algorithm}[h]
  \caption{\textsc{TensorVariableElimination}}
  \label{alg:plated-sumproduct}
  {\bf input} variables $V$, factors $F$,
    edges $E\subseteq V\times F$, \\
    \phantom{\bf input} plate sets $P\!:\!V\cup F\to \mathcal P(B)$, \\
    \phantom{\bf input} plate sizes $M\!:\!B\to\mathbb N$.\\
  {\bf output} $\textsc{PlatedSumProduct}((V,\!F,\!E,\!P),M)$ or   \textbf{error}. \\
  Initialize an empty list of scalars $S\from []$.\\
  \While{$F$ \normalfont{is not empty}}{
    Choose a leaf plate set $L\in\{P(f)\mid f\in F\}$ \\
    \tab with a maximal number of plates.\\
    Let $V_L \from \{v \in V \mid P(v) = L\}$ be the variables in $L$.\\
    Let $F_L\from \{f \in F \mid P(f) = L\}$ be the factors in $L$.\\
    Let $E_L\from E\cap (V_L\times F_L)$ be the edges in $L$.\\
    \For{$(V_c,F_c)$ \textrm{\bf in} \textsc{Partition}$(V_L,F_L,E_L)$}{
      Let $f\from\textsc{SumProduct}(F_c,V_c)$.\\
      Let $V_f\from \{v\mid (v,f)\in E\cap((V\setminus V_c)\times F_c)\}$ \\
      \tab be the set of $f$'s remaining variables.\\
      Remove component $(V_c,F_c)$ from $V,F,E,P$.\\
      \eIf{$V_f$ \normalfont{is empty}}{
        Add $\textsc{Product}(f, L, M)$ to scalars $S$.
      }{
        Let $L'\from\bigcup\{P(v)\mid v\in V_f\}$ be the next\\
        \tab plate set where $f$ has variables.\\
        \lIf{$L'=L$}{\textbf{error}(``Intractable!'')}
        Let $f'\from\textsc{Product}(f, L\setminus L', M)$.\\
        Add $f'$ to $F,E,P$ appropriately.
      }
    }
  }
  \Return $\textsc{SumProduct}(S, \{\})$
\end{algorithm}

To help characterize when Algorithm~\ref{alg:plated-sumproduct} succeeds, we now introduce the concept of a graph minor.
\begin{definition}
A plated graph\footnote{When considering graph minors, we view factor graphs $(V,F,E)$ as undirected graphs $(V\cup F,E)$.} $H$ is a \emph{minor} of the plated graph $G$ if it can be obtained from $G$ by a sequence of edits of the form: deleting a vertex, deleting an edge, deleting a plate, or merging two vertices $u,v$ connected by an edge and in identical plates $P(u)=P(v)$.
\end{definition}
\begin{theorem} \label{thm:success}
Algorithm~\ref{alg:plated-sumproduct} succeeds iff $G$ has no plated graph minor $\left(\{u,v,w\},\, \left\{(u,v),(v,w)\right\},\, P\right)$ where $P(u)=\{a\}$, $P(v)=\{a,b\}$, $P(w)=\{b\}$, $a\ne b$, and $u,w$ both include variables.
\end{theorem}
\begin{proof}
See Appendix~\ref{sec:proof-success}.
\end{proof}
This plated graph minor exclusion property essentially excludes the RBM Example 3.2 above, which is a minimal example of an intractable input.%
\footnote{See Appendix~\ref{sec:walkthrough-1} for a detailed walk-through of Algorithm~\ref{alg:plated-sumproduct} on this model, leading to \textbf{error}.}
If Algorithm~\ref{alg:plated-sumproduct} fails, one could fall back to unrolling a plate and continuing (at cost exponential in plate size); we leave this for future research.

\subsection{Complexity of tensor variable elimination} \label{sec:complexity}


It is well known that message passing algorithms have computational complexity exponential in the treewidth of the input factor graph but only linear in the number of variables \cite{chandrasekaran2012complexity,kwisthout2010necessity}.
In this section we generalize this result to the complexity of plated message passing on tensor factors.
We show that for an easily identifiable class of plated factor graphs, serial complexity is polynomial in the tensor sizes of the factors, and parallel complexity is sublinear in the size of each plate.
Essentially this characterizes when the tensor size of one factor of a plated factor graph determines the treewidth of the unrolled non-plated factor graph: in the polynomial case, treewidth is independent of tensor size.

\begin{example} \label{ex:tractable-eqn}
Consider the plated factor graph of Example~\ref{ex:tractable} with nested plates.
We wish to compute,
\begin{align*}
&\sum_x\sum_{y_1}\cdots\sum_{y_I}
F_{x}
\Bigl[\prod_i G_{i,y_i}\Bigr]
\Bigl[\prod_{i,j} H_{i,j,x,y_i}\Bigr]
\\&\hspace{2cm}=\sum_x F_{x}
\prod_i \sum_{y_i} G_{i,y_i} \prod_j H_{i,j,x,y_i}
\end{align*}
where we are able to commute the sums over $y_i$ inside the product over $i$, thereby reducing an algorithm of cost exponential in $I$ to a polynomial-cost algorithm.
\end{example}

\begin{example} \label{ex:intractable-eqn}
Consider the plated factor graph of Example~\ref{ex:intractable} with overlapping, non-nested plates.
Although the plated factor graph is a tree of tensors, the unrolled factor graph has treewidth $O(I+J)$; hence complexity will be exponential in the tensor sizes $I,J$ of dimensions $i,j$.

We can reach the same conclusion from sum-product,
\begin{align*}
&\sum_{x_1}\cdots\sum_{x_I}
 \sum_{y_1}\cdots\sum_{y_J}
 \prod_{i,j} F_{i,j,x_i,y_j}
\\&\hspace{2cm}=
 \sum_{x_1}\cdots\sum_{x_I}
 \prod_{i,j}
 \sum_{y_j}F_{i,j,x_i,y_j}
\end{align*}
Here we cannot commute both Cartesian product summations inside the product and so the computation is necessarily exponential in $I,J$.
\end{example}

We now show Algorithm~\ref{alg:plated-sumproduct} accepts the largest class of plated factor graphs for which a polynomial time strategy exists.

\begin{theorem} \label{thm:complexity}
Let $G=(V,\,F,\, E,\, P\!:\!V\cup F\to\mathcal P(B))$ be a plated factor graph.
Assume variable domains have nontrivial sizes ${|\dom(v)|\ge 2,\,\forall v\in V}$.
Then Algorithm~\ref{alg:plated-sumproduct} succeeds on $G$ iff $\textsc{PlatedSumProduct}(G,M)$ can be computed with complexity polynomial in plate sizes $M\!:\!B\to\mathbb N$.
\end{theorem}
\begin{proof}[Proof sketch]
(see Appendix~\ref{sec:proof-complexity} for full proof).

($\Rightarrow$) Algorithm~\ref{alg:plated-sumproduct} has complexity polynomial in $M$.

($\Leftarrow$)
Appeal to \cite{kwisthout2010necessity} (which assumes the Exponential Time Hypothesis) to show that the unrolled factor graph must have uniformly bounded treewidth.
Apply Ramsey theory arguments to show there is a single ``plated junction tree'' with certain properties.
Show this tree can only exist if $F$ excludes the minor of Thm.~\ref{thm:success}, hence Algorithm~\ref{alg:plated-sumproduct} succeeds.
\end{proof}

When a plated factor graph is asymptotically tractable according to Thm.~\ref{thm:complexity}, it is also tractable on a parallel machine.

\begin{theorem} \label{thm:parallel}
If Algorithm~\ref{alg:plated-sumproduct} runs in sequential time $T$ when $M(b)\!=\!1,\,\, \forall b$, then it runs in time $T+O\left(\Sigma_b\log M(b)\right)$ on a parallel machine with $\Pi_b M(b)$ processors and perfect efficiency.
\end{theorem}
\begin{proof}
Algorithm~\ref{alg:plated-sumproduct} depends on $M$ only through calls to \textsc{SumProduct} and \textsc{Product}.
\textsc{SumProduct} operates independently over plates, and hence parallelizes with perfect efficiency.
\textsc{Product} reductions over each plate $b$ incur only $O(\log M(b))$ time on a parallel machine.
\end{proof}
This result suggests that (tractable) plated factor graphs are a practical model class for modern hardware. See Sec.~\ref{sec:performance} for empirical verification of the computational complexity described in Thm.~\ref{thm:parallel}.

While tensor variable elimination enables parallelism without increasing arithmetic operation count, it also reduces 
overhead involved in graph manipulation.
A significant portion of runtime in \textsc{SumProduct} is spent on constructing a junction tree.
By exploiting repeated structure, tensor variable elimination can restrict its junction tree problems to much smaller factor graphs (constructed locally for each connected component of each plate set), leading to lower overhead and the opportunity to apply better heuristics.%
\footnote{E.g.~opt\_einsum uses an optimal strategy for factor graphs with up to four variables.}

\subsection{Generic tensor variable elimination}
\label{sec:adjoint}

Because Algorithm~\ref{alg:plated-sumproduct} is generic\footnote{In the sense of generic programming \cite{musser1988generic}} in its two operations (plus and multiply),
it can immediately be repurposed to yield other algorithms on plated factor graphs, for example a polynomial-time \textsc{PlatedMaxProduct} algorithm that generalizes the \textsc{MaxProduct} algorithm on factor graphs.

Further extensions can be efficiently formulated as adjoint algorithms,
which proceed by recording an adjoint compute graph alongside the forward computation and then traversing the adjoint graph backwards starting from the final result of the forward computation \cite{darwiche2003differential,eisner2016inside,azuma2017algebraic,belle2016semiring}.
These extensions include: computing marginal distributions
of all variables (generalizing the forward-backward algorithm);
maximum a posteriori estimation (generalizing Viterbi-like algorithms);
and drawing joint samples of variables (generalizing the forward-filter backward-sample algorithm).
Note that while most message passing algorithms need only assume that the sum and product operations have semiring structure \cite{kohlas2008semiring,khamis2016faq}, marginal computations additionally require division.

\section{Application to Probabilistic Programming}
\label{sec:applications}

Plated factor graphs provide a general-purpose intermediate representation for many applications in probabilistic modeling requiring efficient inference and easy parallelization.
To make tensor variable elimination broadly usable for these applications, we integrate it into two frameworks\footnote{Open-source implementations are available; see  \texttt{http://docs.pyro.ai/en/dev/ops.html}} and use both frameworks in our experiments.






\begin{figure}
    \centering
\resizebox {.45\columnwidth} {!} {
\begin{tikzpicture}
\tikzstyle{factor} = [square, draw, rounded corners]
\definecolor{mygrey}{RGB}{220,220,220}
\node[latent, xshift=1cm](X) {$X$};
\node[latent, fill=mygrey, xshift=2.5cm](Z) {$Z$};
\node[latent,  xshift=4cm](Y) {$Y$};
\plate [inner xsep=0.3cm, inner ysep=0.2cm, yshift=0.1cm] {a}{(Z)}{};
\node [right plate caption=a-wrap, yshift=0cm, xshift=0.2cm]{$J$};
\plate [inner sep=0.4cm, yshift=0.1cm, minimum height=1.9cm, minimum width=3.1cm, xshift=-0.1cm] {b}{(Z)(Y)}{};
\node [right plate caption=b-wrap, yshift=-0.2cm,xshift=.2cm] {$I$};
\draw[->] (X) -> (Z);
\draw[->] (Y) -> (Z);
\end{tikzpicture}
} 
\hspace{2pt}
\resizebox {.45\columnwidth} {!} {
\begin{tikzpicture}
\tikzstyle{factor} = [square, draw, rounded corners]
\definecolor{mygrey}{RGB}{220,220,220}
\node[latent, xshift=1cm](X) {$X$};
\node[latent, fill=mygrey, xshift=2.5cm](Z) {$Z$};
\node[latent,  xshift=4cm](Y) {$Y$};
\plate [inner xsep= 0.4cm, inner ysep=0.2cm, yshift=0.0cm] {a}{(X)(Z)}{};
\node [left plate caption=a-wrap, yshift=-0.05cm]{$I$};
\plate [inner xsep= 0.4cm, inner ysep=0.2cm, yshift=0.2cm] {b}{(Z)(Y)}{};
\node [right plate caption=b-wrap, yshift=1cm, xshift=.2cm]{$J$};
\draw[->] (X) -> (Z);
\draw[->] (Y) -> (Z);
\end{tikzpicture}
} 
        \caption{The tractable and intractable plated factor graphs in Examples~\ref{ex:tractable} and ~\ref{ex:intractable} arise from the plated graphical model on the left and right, respectively.  Here the two random variables $X,Y$ are unobserved and the random variable $Z$ is observed.        }
\label{fig:directed-examples}
\end{figure}
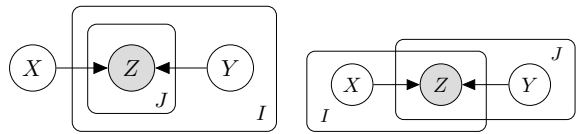

\subsection{Plated probabilistic programs in Pyro} \label{sec:pplsyntax}

First we integrate our implementation into the Pyro probabilistic programming language \cite{bingham2018pyro}.
This allows us to specify discrete latent variable models easily, programmatically constructing complex distributions and explicitly indicating repeated structure.
The syntax relies on a \lstinline$plate$ context manager.\footnote{We also introduce a \lstinline$markov$ context manager to deal with markov structure.}
Inside these plates,~\texttt{sample} statements are batched and assumed to be conditionally independent along the plate dimension. 
\begin{example} \label{ex:tractable-pyro}
  The directed graphical model in Fig.~\ref{fig:directed-examples} (left) can be specified by the Pyro program (see Appendix~\ref{sec:pyro-plate})
\begin{lstlisting}[language=Python]
def model(z):
  I, J = z.shape
  x = pyro.sample("x", Bernoulli(Px))
  with pyro.plate("I", I, dim=-2):
    y = pyro.sample("y", Bernoulli(Py))
    with pyro.plate("J", J, dim=-1):
      pyro.sample("z", Bernoulli(Pz[x,y]), 
                  obs=z)
\end{lstlisting}
\end{example}
To extract plated factor graphs from such programs without using static analysis,
we use a nonstandard interpretation of Pyro \lstinline$sample$ statements \cite{wingate2011nonstandard} as vectorized enumeration over each distribution's support.
That is, when running a program forward we create a  tensor at each sample site that lists all possible values of each distribution rather than draw random samples.
To avoid conflict among multiple enumerated variables, we dynamically assign each variable a distinct tensor dimension along which its values are enumerated, relying on array broadcasting \cite{walt2011numpy} to correctly combine results of multiple nonstandard sample values (as in \lstinline$Pz[x,y]$ above).



\subsection{Undirected graphical models using \lstinline$einsum$}
\label{sec:einsum}

A standard interface for expressing inference in undirected graphical models is \lstinline$einsum$ notation. Basic \lstinline$einsum$ computes the product of a collection of tensors, with some dimensions preserved in the final result and remaining dimensions summed out. This operation behaves analogously 
to $\textsc{SumProduct}$, with the inputs to \lstinline$einsum$ given as: a string denoting the tensor dimensions of all input and output tensors, and a list of tensors. The former gives the factor graph topology and the latter the factors themselves.
For example, \lstinline$einsum("xy,yz->xz",F,G)$ corresponds to matrix multiplication
for matrices \lstinline$F,G$.
This \lstinline$einsum$ expression also corresponds to reducing the $y$ variable in a factor graph with three nodes and two factors.

Plated \lstinline$einsum$ generalizes \lstinline$einsum$ notation, e.g.
\vspace*{-.25cm}
\begin{Verbatim}[fontsize=\small]
 einsum("xy,iyz->xz",F,G,plates="i") ==
 einsum("xy,yz,yz,yz->xz",F,G[0],G[1],G[2])
\end{Verbatim}
\vspace*{-.25cm}
where dimension \lstinline$x,z$ are preserved, variable \lstinline$y$ is sum-reduced, and plate \lstinline$i$ is product-reduced.
More generally, factors can include plated variables, e.g. Example~\ref{ex:tractable} can be written as
\vspace*{-.25cm}
\begin{Verbatim}[fontsize=\small]
 einsum("x,iy,ijxy->",F,G,H,plates="ij") ==
 einsum("x,y,z,xy,xy,xy,xz,xz,xz->",
        F,G[0],G[1],H[0,0],H[0,1],
        H[0,2],H[1,0],H[1,1],H[1,2])
\end{Verbatim}
\vspace*{-.25cm}
where the rightmost dimension of \lstinline$G$ and \lstinline$H$ denotes a distinct variable for each slice \lstinline$i$ (\lstinline$y$ and \lstinline$z$ in the unrolled version).
Thus the effective number of variables grows with the size of plate \lstinline$i$, but the plated notation requires only a constant number of symbols.
Formally we infer each variable's plate set $P(v)$ as the largest set consistent with the input string. In addition, this version 
of \lstinline$einsum$ implements generic tensor variable elimination, and so it can also be used to compute marginals and draw samples with the same syntax. 

\section{Experiments} \label{sec:experiments}

We experiment with plated factor graphs as a modeling language 
for three tasks: polyphonic music prediction, animal movement modeling 
and latent sentiment analysis.\footnote{See the supplementary materials for plate diagrams for
the models in Sec.~\ref{sec:hmmexp} and Sec.~\ref{sec:mehmmexp}.} Experiments consider different variants 
of discrete probabilistic models and their combination with neural neworks.

\subsection{Hidden Markov Models with Autoregressive Likelihoods}
\label{sec:hmmexp}
In our first experiment we train variants of a hidden Markov model\footnote{For an introduction see e.g.~reference \cite{ghahramani2001introduction}.} (HMM) on a polyphonic music modeling task \cite{boulanger2012modeling}. The data consist of sequences $\{\by_1, ..., \by_T \}$, where each $\by_t \in \{0, 1\}^{88}$ denotes the presence or absence of 88 distinct notes. 
This task is a common challenge for latent variable models such as continuous state-space models (SSMs), where one of the difficulties of inference is that training is typically stochastic, i.e. the latent variables need to be sampled.\footnote{This is true for all five reference models in this section, apart from the RTRBM.} In contrast, the discrete latent variable models explored here---each defined  through a plated
factor graph---admit efficient tensor variable elimination, obviating the need for sampling.

We consider 12 different latent variable models, where the emission likelihood for each note is replicated on a plate of size 88, so that notes at each time step are conditionally independent.
Writing these models as probabilistic programs allows us to easily experiment with dependency structure and parameterization, with variation along two dimensions:
\begin{enumerate}[nolistsep]
\item the dependency structure of the discrete latent variables
\item whether the likelihood $p(\by_t | \cdot)$ is autoregressive, i.e.~whether it depends on $\by_{t-1}$, and if so whether $p(\by_t | \cdot)$ is parameterized by a neural network
\end{enumerate}
\newcommand{\modelhmmjsb}{\small 8.28}
\newcommand{\modelfhmmjsb}{\small 8.40}
\newcommand{\modelpfhmmjsb}{\small 8.30}
\newcommand{\modelsohmmjsb}{\small 8.70}
\newcommand{\modelarhmmjsb}{\small 8.00}
\newcommand{\modelarfhmmjsb}{\small 8.22}
\newcommand{\modelarpfhmmjsb}{\small 8.39} 
\newcommand{\modelarsohmmjsb}{\small 8.19} 
\newcommand{\modelnnhmmjsb}{\small \bf 6.73}
\newcommand{\modelnnfhmmjsb}{\small 6.86}
\newcommand{\modelnnpfhmmjsb}{\small 7.07}
\newcommand{\modelnnsohmmjsb}{\small 6.78}
\newcommand{\modelhmmnottingham}{\small 4.49}
\newcommand{\modelfhmmnottingham}{\small 4.72}
\newcommand{\modelpfhmmnottingham}{\small 4.76}
\newcommand{\modelsohmmnottingham}{\small 4.96}
\newcommand{\modelarhmmnottingham}{\small 3.29}
\newcommand{\modelarfhmmnottingham}{\small 3.57}
\newcommand{\modelarpfhmmnottingham}{\small 4.82}
\newcommand{\modelarsohmmnottingham}{\small 3.34}
\newcommand{\modelnnhmmnottingham}{\small \bf 2.67}
\newcommand{\modelnnfhmmnottingham}{\small 2.82}
\newcommand{\modelnnpfhmmnottingham}{\small 2.81}
\newcommand{\modelnnsohmmnottingham}{\small 2.81}
\newcommand{\modelhmmpiano}{\small 9.41}
\newcommand{\modelfhmmpiano}{\small 9.55}
\newcommand{\modelpfhmmpiano}{\small 9.49}
\newcommand{\modelsohmmpiano}{\small 9.57}
\newcommand{\modelarhmmpiano}{\small 7.30}
\newcommand{\modelarfhmmpiano}{\small 7.36}
\newcommand{\modelarpfhmmpiano}{\small 9.57}
\newcommand{\modelarsohmmpiano}{\small \bf 7.11}
\newcommand{\modelnnhmmpiano}{\small 7.32}
\newcommand{\modelnnfhmmpiano}{\small 7.41}
\newcommand{\modelnnpfhmmpiano}{\small 7.47}
\newcommand{\modelnnsohmmpiano}{\small 7.29}
\begin{table}[t!]
\begin{center}
\resizebox {.6\columnwidth} {!} {
    \begin{tabu}{|c|[1pt]c|c|c|}    \hline
   \cellcolor[gray]{0.75} & \multicolumn{3}{c|}{\small Dataset \cellcolor[gray]{0.95}} \\  \hline
    \small Model \cellcolor[gray]{0.95} &  \small JSB & \small Piano  & \small Nottingham  \\  \tabucline[1pt]{-}
    \small \texttt{HMM} & \modelhmmjsb & \modelhmmpiano & \modelhmmnottingham \\ \hline
    \small \texttt{FHMM} & \modelfhmmjsb & \modelfhmmpiano & \modelfhmmnottingham  \\ \hline
    \small \texttt{PFHMM} & \modelpfhmmjsb & \modelpfhmmpiano & \modelpfhmmnottingham   \\
    \hline
    \small \texttt{2HMM} & \modelsohmmjsb & \modelsohmmpiano & \modelsohmmnottingham   \\
\tabucline[1pt]{-}
    \small \texttt{arHMM} & \modelarhmmjsb & \modelarhmmpiano & \modelarhmmnottingham   \\ \hline
   \small  \texttt{arFHMM} & \modelarfhmmjsb & \modelarfhmmpiano & \modelarfhmmnottingham  \\ \hline
    \small \texttt{arPFHMM} & \modelarpfhmmjsb & \modelarpfhmmpiano & \modelarpfhmmnottingham   \\
    \hline
    \small \texttt{ar2HMM} & \modelarsohmmjsb & \modelarsohmmpiano & \modelarsohmmnottingham   \\
\tabucline[1pt]{-}
    \small \texttt{nnHMM} & \modelnnhmmjsb & \modelnnhmmpiano & \modelnnhmmnottingham  \\ \hline
   \small \texttt{nnFHMM} & \modelnnfhmmjsb & \modelnnfhmmpiano & \modelnnfhmmnottingham   \\ \hline
    \small \texttt{nnPFHMM} & \modelnnpfhmmjsb & \modelnnpfhmmpiano & \modelnnpfhmmnottingham \\ \hline
    \small \texttt{nn2HMM} & \modelnnsohmmjsb & \modelnnsohmmpiano & \modelnnsohmmnottingham \\ 
    \hline
    \end{tabu}
} 
\end{center}
     \caption{Negative log likelihoods for HMM variants on three polyphonic music test datasets; lower is better. See Sec.~\ref{sec:hmmexp} for details.}
\label{table:hmm}
\end{table}
Dependency structures include a vanilla HMM (\texttt{HMM}); two variants of a Factorial HMM (\texttt{FHMM} \& \texttt{PFHMM}) \cite{ghahramani1996factorial}; and a second-order HMM (\texttt{2HMM}).\footnote{For the second-order HMM we use a parsimonious Raftery parameterization of the transition probabilities \cite{raftery1985model}.}
The models denoted by \texttt{arXXX} and \texttt{nnXXX} include an autoregressive likelihood: the former are explicitly parameterized with a conditional probability table, while the latter use a neural network to parameterize the likelihood.
(See the supplementary materials for detailed descriptions.)

We report our results in Table~\ref{table:hmm}. We find that our ability to iterate over a large class of models\footnote{See \texttt{\small https://git.io/fjc82} for a reference implementation for a selection of HMM variants.}---in particular different latent dependency structures, each of which requires a different message passing algorithm---proved useful, since different datasets preferred different classes of models. Autoregressive models yield the best results, and factorial HMMs perform worse
across most models.

We note that these classic HMM variants (upgraded with neural network likelihoods) are competitive with baseline state space models, outperforming STORN \cite{bayer2014learning} on 3 of 3 datasets,  LV-RNN \cite{gu2015neural} on 2 of 3 datasets, Deep Markov Model \cite{krishnan2017structured} on 2 of 3 datasets,
 RTRBM \cite{sutskever2009recurrent,boulanger2012modeling} on 1 of 3 datasets, and 
 TSBN \cite{gan2015deep} on 3 of 3 datasets.

\subsection{Hierarchical Mixed-Effect Hidden Markov Models}
\label{sec:mehmmexp}
In our second set of experiments, we consider models for describing the movement of populations of wild animals. 
Recent advances in sensor technology have made it possible to capture the movements of multiple animals in a population at high spatiotemporal resolution \cite{mcclintock2013combining}.
Time-inhomogeneous discrete SSMs, where the latent state encodes an individual's behavior state (like ``foraging'' or ``resting'') and the state transition matrix at each timestep is computed with a hierarchical discrete generalized linear mixed model, have become popular tools for data analysis thanks to their interpretability and tractability \cite{zucchini2016hidden,mcclintock2018momentuhmm}.

Rapidly iterating over different variants of such models, with nested plates and hierarchies of latent variables that couple large groups of individuals within a population, is difficult to do by hand
but can be substantially simplified by expressing models as plated probabilistic programs and performing inference with tensor variable elimination.\footnote{See \texttt{https://git.io/fjc8a} for a reference implementation.}

To illustrate this, we implement a version of the model selection process for movement data from a colony of harbour seals in the United Kingdom described in \cite{mcclintock2013combining},
fitting three-state hierarchical discrete SSMs with no random effects (\texttt{No RE}, a vanilla HMM), sex-level discrete random effects (\texttt{Group RE}), individual-level discrete random effects (\texttt{Individual RE}), and both sex- and individual-level discrete random effects (\texttt{Individual+Group RE}).
See the supplement for details on the dataset, models and training procedure.

\newcommand{\sealgnin}{\small $353\times 10^3$}
\newcommand{\sealgnid}{\small $341 \times 10^3$}
\newcommand{\sealgdin}{\small $342 \times 10^3$}
\newcommand{\sealgdid}{\small $341 \times 10^3$}
\begin{table}[t!]
\begin{center}
\resizebox {.55\columnwidth} {!} {

\begin{tabu}{|c|[1pt]c|}
    \hline
    \small Model \cellcolor[gray]{0.95} &  \small AIC \\  \tabucline[1pt]{-}

    \small \texttt{No RE (HMM)} & \sealgnin   \\ \hline
    \small \texttt{Individual RE} & \sealgnid \\ \hline
    \small \texttt{Group RE} & \sealgdin \\ \hline
    \small \texttt{Individual+Group RE} & \sealgdid \\ 
    \hline
\end{tabu}
} 
\end{center}
     \caption{Akaike Information Criterion (AIC) scores for hierarchical mixed effect HMM variants fit with maximum likelihood on animal movement data. Lower is better. See Sec.~\ref{sec:mehmmexp}.}
\label{tbl:mehmm}
\end{table}

We report AIC scores for all models in Table~\ref{tbl:mehmm}.
Although our models do not exactly match those in the original analysis,\footnote{See supplement for a discussion of differences.} our results support theirs in suggesting that including individual-level random effects is essential because there is significant behavioral variation across individuals and sexes that is unexplained by the available covariates.

\subsection{Latent Variable Classification}
\label{sec:sentclass}
We next experiment with model flexibility by designing a conditional random field \cite{lafferty2001crf}
model for latent variable classification on a sentiment classification task. Experiments use the Sentihood dataset \cite{sentihood}, which consists of sentences containing named location entities with sentiment labels along different aspects.
For a sentence $\mathbf{x} = \langle x_1,\ldots,x_T\rangle$ labels are tuples $(a, l, y)$ that contain  an aspect $a \in {\cal A} =  \left\{\textrm{general},\textrm{safety},\ldots\right\}$,
a location $l \in \cal{L} = \left\{\texttt{Location1},\texttt{Location2}\right\}$, and a sentiment $y \in \left\{\textrm{positive}, \textrm{negative}, \textrm{none}\right\}$.
The task is to predict the sentiment of a sentence given a location and aspect, for example $p(y \mid  \mathbf{x}, a=\text{price}, l=\text{\texttt{Location1}})$.


Standard approaches to this sentence-level classification task use neural network models to directly make sentence-level predictions. We instead propose a latent variable 
approach that explicitly models the sentiment of each word with respect to all locations and aspects. 
This approach can provide clearer insight into the specific 
reasoning of the model and also permits conditional inference.

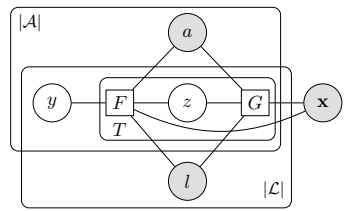
\begin{figure}[t]
\begin{center}
\resizebox {.55\columnwidth} {!} {
\begin{tikzpicture}
\node[latent] (Y) {$y$};
\node[factor, xshift=1.25cm] (FYS) {$F$};
\node[latent, xshift=2.5cm](S) {$z$};
\node[factor, xshift=3.75cm](FS) {$G$};
\node[xshift=1.25cm, yshift=-0.5cm]() {$T$};
\node[obs, xshift=5cm](X) {$\mathbf{x}$};
\node[obs, xshift=2.5cm, yshift=1.3cm] (A) {$a$};
\node[obs, xshift=2.5cm, yshift=-1.45cm] (L) {$l$};

\plate [inner sep=0.1cm, xshift=0cm, yshift=0.0cm] {t} {(FS)(S)(FYS)} {};
\plate [inner xsep= 0.3cm, inner ysep= 0.2cm, xshift=-0.1cm, yshift=-0.1cm] {a} {(Y)(FYS)(S)(A)(FS)} {};
\plate [inner xsep= 0.3cm, inner ysep= 0.1cm, xshift=0.1cm, yshift=0.2cm] {l} {(Y)(FYS)(S)(L)(FS)} {};
\node[caption, below left=-0.1cm and -0.3cm of a-wrap.north west] {$|\cal{A}|$};
\node[caption, below left=-0.5cm and -0.4cm of l-wrap.south east] {$|\cal{L}|$};

\draw (Y) -- (FYS);
\draw (X) -- (FS);
\draw[-] (X) to [bend left=25] (FYS);
\draw (FYS) -- (S);
\draw (S) -- (FS);
\draw (FS) -- (A);
\draw (FS) -- (L);
\draw[-] (FYS) to [] (A);
\draw[-] (FYS) to [] (L);
\end{tikzpicture}
} 
\end{center}
\caption{The graphical model for the sentiment analysis CRF in Sec.~\ref{sec:sentclass}. The aspect $a$ and location $l$ are observed, while the word-level sentiments $z_t$ and sentence-level sentiment $y$ for the particular aspect and location pair must be inferred.  See the supplementary material for the exact parameterization of the factors.}
\label{fig:sentcrf}
\end{figure}

Our conditional random field model is represented as a plated factor graph in Figure~\ref{fig:sentcrf}.
Here $\mathbf{z} = \langle z_1,\ldots,z_T\rangle$ is the latent word-level sentiment for fixed $l$ and $a$.
The two plated factors $G$ and $F$ represent the word aspect-location-sentiment potentials and the word-sentence potentials, respectively. 
We parameterize these factors with a bidirectional LSTM (BLSTM) over word embeddings of $\mathbf{x}$ whose initial state is given by an embedding of the location $l$ and aspect $a$.
We experiment with three variants:\footnote{See \texttt{https://github.com/justinchiu/sentclass} for a reference implementation.} 1) \texttt{CRF-LSTM-Diag} parameterizes $G$ with the output of the BLSTM and $F$ with a diagonal matrix; 2) \texttt{CRF-LSTM-LSTM} parameterizes both $G$ and $F$ with the output of the BLSTM; and 3) \texttt{CRF-Emb-LSTM} parameterizes $G$ with word embeddings and $F$ with the output of the BLSTM. As a baseline we reimplement the direct BLSTM model \texttt{LSTM-Final} model from \citet{sentihood}. See the supplementary material for full details.


The accuracy of our models is given in 
Table ~\ref{tbl:sentres}.\footnote{The results of previous work can be found summarized succinctly in \citet{rentnet}. Note that our goal is not to improve upon previous results (other models have higher accuracy); rather we aim to capitalize on the plated representation to infer latent word-level sentiment.} The CRF models all achieve similar performance to the baseline \texttt{LSTM-Final} model. However, as a result of the factor graph representation, we demonstrate the ability to infer word-level sentiment conditioned on a sentence-level sentiment as well as sparsity of the conditional word-level sentiments for the \texttt{CRF-Emb-LSTM} model. We provide an example sentence and the inferred conditional word sentiments for the \texttt{CRF-Emb-LSTM} model in Fig.~\ref{fig:condwordsent}.

\newcommand{\lracc}{\small 0.875}
\newcommand{\lrf}{\small 0.393}
\newcommand{\lstmacc}{\small 0.820}
\newcommand{\lstmf}{\small 0.689}
\newcommand{\senticacc}{\small 0.893}
\newcommand{\senticf}{\small 0.782}
\newcommand{\rentnetacc}{\small\bf 0.910} 
\newcommand{\rentnetf}{\small 0.785} 
\newcommand{\blstmacc}{\small 0.821}
\newcommand{\blstmf}{\small 0.780}
\newcommand{\crfnbacc}{\small 0.805}
\newcommand{\crfnbf}{\small 0.764}
\newcommand{\crfnblacc}{\small 0.843}
\newcommand{\crfnblf}{\small 0.799}
\newcommand{\crfacc}{\small 0.833}
\newcommand{\crff}{\small 0.779}


\begin{table}[t!]
\begin{center}
\resizebox {.49\columnwidth} {!} {

    \begin{tabu}{|c|[1pt]c|c|}
    \hline
    \cellcolor[gray]{0.75} & \multicolumn{2}{c|}{\small Metric \cellcolor[gray]{0.95}} \\
    \hline
    \small Model \cellcolor[gray]{0.95} &  \small Acc & \small F1 \\  \tabucline[1pt]{-}

   \small \texttt{LSTM-Final} & \blstmacc & \blstmf   \\ \hline
    \small \texttt{CRF-LSTM-Diag} & \crfnbacc & \crfnbf \\ \hline
    \small \texttt{CRF-LSTM-LSTM} & \crfnblacc & \crfnblf \\ \hline
    \small \texttt{CRF-Emb-LSTM} & \crfacc & \crff \\ 
    \hline
\end{tabu}
} 
\end{center}
     \caption{Test sentiment accuracies and aspect F1 scores for sentiment classification models on Sentihood.  Higher is better for both metrics. Both metrics are calculated by ignoring examples with gold labelled `none' sentiments; see Sec.~\ref{sec:sentclass} for details.}
\label{tbl:sentres}
\end{table}

\begin{figure}[t]
\centering
\small
\resizebox {.61\columnwidth} {!} {
\pgfplotstabletypeset[
    color cells={min=0,max=2.5},
    col sep=&,
    row sep=\\,
    /pgfplots/colormap={whiteblue}{rgb255(0cm)=(255,255,255); rgb255(1cm)=(0,0,164)},
    every head row/.style={
        before row={\hline},
        after row=\hline,
    },
    every last row/.style={
        after row=\hline,
    },
    columns/convenient/.style = {column type=c|},
    columns/sentiment/.style = {reset styles, string type, column type=|c|, column name = $y$},
    columns/ellipsis/.style = {column name = $\ldots$},
    fixed,
    skip 0.,
]{
    sentiment & Loc1 & ellipsis & but & not & too & convenient\\
    \Neutrey  & 0 & 1& .89 & .98 & .84 & 0\\
    \Smiley   & 0 & 0& 0   & 0   & 0   & .31\\
    \Sadey    & 1 & 0& .11 & .02 & .16 & .69\\
} }
\caption{The inferred word-level sentiment $z_t$ for aspect \textit{transit-location} and location \texttt{Loc1} conditioned on a \textit{negative} sentence-level sentiment. The ellipsis contains the elided words `is a great place to live', all of which get \textit{none} sentiment. The model assigns the \textit{none} sentiment to most words, preventing them from influencing the polarity of the sentence-level sentiment.}
\label{fig:condwordsent}
\end{figure}

\subsection{Performance Evaluation} \label{sec:performance}

To measure the performance of our implementation we use the benchmark plated model%
\footnote{See Appendix~\ref{sec:walkthrough-3} for a detailed walkthrough of Algorithm~\ref{alg:plated-sumproduct} on this model.}
of Fig.~\ref{fig:perf-model}.
We compute \textsc{PlatedSumProduct} and four different adjoint operations: gradient, marginal, sample, and MAP.
Figure~\ref{fig:runtime} shows results obtained on an Nvidia Quadro P6000 GPU.
We find that runtime is approximately constant until the GPU is saturated (at $I\times J \approx 10^4$), and runtime is approximately linear in $I\times J$ for larger plate sizes, empirically validating Thm~\ref{thm:parallel}.


\begin{figure}[ht!]
\begin{center}
\resizebox {.70\columnwidth} {!} {
\begin{tikzpicture}
\node[factor] (fx) {};
\node[latent, xshift=1.25cm] (x) {$X$};
\node[factor, xshift=2.5cm] (fyx) {};
\node[latent, xshift=4.5cm] (y) {$Y$};

\node[factor, xshift=1.25cm, yshift=0.8cm] (fwx) {};
\node[factor, xshift=4.5cm, yshift=0.8cm] (fzy) {};

\node[latent, xshift=1.25cm, yshift=1.6cm] (w) {$W$};
\node[factor, xshift=2.5cm, yshift=1.6cm] (fvw) {};
\node[latent, xshift=3.5cm, yshift=1.6cm] (v) {$V$};
\node[latent, xshift=4.5cm, yshift=1.6cm] (z) {$Z$};

\plate[inner xsep=0.3cm, xshift=-0.1cm, yshift=0.2cm]
    {a} {(fwx)(w)(fvw)(v)(z)(fzy)} {};
\node[caption, below left=-.3cm and -.1cm of a-wrap.north west] {$I$};

\plate[inner xsep=0.5cm, inner ysep=0.1cm, xshift=0.1cm, yshift=0cm]
    {b} {(fvw)(v)(z)(fzy)(y)(fyx)} {};
\node[caption, below left=-.15cm and -.5cm of b-wrap.south east] {$J$};

\draw (fx) -- (x);
\draw (x) -- (fyx);
\draw (fyx) -- (y);

\draw (x) -- (fwx);
\draw (fwx) -- (w);
\draw (y) -- (fzy);
\draw (fzy) -- (z);

\draw (w) -- (fvw);
\draw (fvw) -- (v);
\end{tikzpicture}
} 
\end{center}
\caption{
The tractable model used in our performance benchmark.
Note that this model would be intractable if there were a factor directly connecting variables $V$ and $Z$.
}
\label{fig:perf-model}
\end{figure}
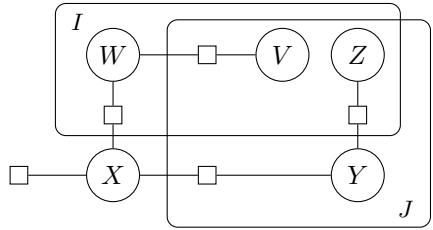

\begin{figure}[ht!]
\vskip 0.2in
\begin{center}
\centerline{
\includegraphics[width=.75\columnwidth]{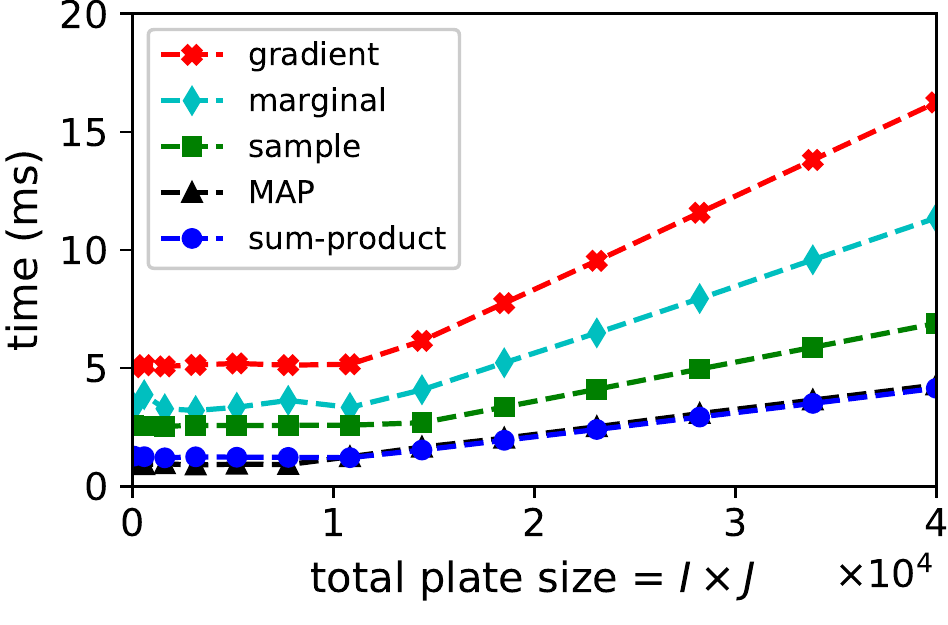}}
\caption{
Runtime of Algorithm~\ref{alg:plated-sumproduct} as we vary plate size (where $I=J$) and fix domain size $|\dom(v)|=32$ for all variables.
}
\label{fig:runtime}
\end{center}
\vskip -0.2in
\end{figure}

\section{Conclusion} \label{sec:conclusion}
This work argues for plated factor graphs as an intermediate representation that 
preserves shared structure for calculations with discrete random variables, and develops 
a tensor variable elimination algorithm for efficiently computing plated sum-product 
and related factor graph queries for a subset of plated models. We show how this approach 
can be used to compute key quantities for directed and undirected graphical models and demonstrate
its implementation as a general purpose intermediary representation for probabilistic programming with 
discrete random variables. Applications further demonstrate that this provides a simple, flexible, and 
efficient framework for working with discrete graphical models, and that these models can often outperform 
more complicated counterparts, while providing interpretable latent representations. The work itself 
is integrated into a widely used probabilistic programming system, and we hope it provides a framework 
for experiments incorporating discrete variable models into large-scale systems.

\section*{Acknowledgements}

We would like to thank many people for early discussions and feedback, including JP Chen, Theofanis Karaletsos, Ruy Ley-Wild, and Lawrence Murray.
MJ would like to thank Felipe Petroski Such for help with infrastructure for efficient distribution of experiments.
AMR and JC were supported by NSF \#1704834 and \#1845664.

\bibliography{enumeration}
\bibliographystyle{icml2019}


\newpage
\newpage

\appendix

\section{Proofs} \label{sec:proofs}

\subsection{Proof of Theorem~\ref{thm:success}} \label{sec:proof-success}

Note that while the ($\Leftarrow$) direction would follow from Theorem~\ref{thm:tfae} below, we provide an independent proof that does not rely on the Exponential Time Hypothesis.

\begin{proof}[Proof of Theorem~\ref{thm:success}]
($\Rightarrow$)
Suppose by way of contradiction that Algorithm~\ref{alg:plated-sumproduct} has failed.
Since $L=L'=\bigcup\{v\in V_f\mid v\in P(v)\}$ but $\forall v\in V_f,\, P(v)\subsetneq L$, there must be at least two distinct $u,w\in V_f$ and two distinct plates $a\ne b\in L$ such that $a\in P(u)\setminus P(w)$ and $b\in P(w)\setminus P(u)$.
Letting $v=f$ with $a,b\in P(v)$ provides the required graph minor $\left(\{u,v,w\},\, \left\{(u,v),(v,w)\right\},\, P\right)$.

($\Leftarrow$)
Suppose by way of contradiction that $G$ has such a graph minor $\left(\{u,v,w\},\, \left\{(u,v),(v,w)\right\},\, P\right)$.
Algorithm~\ref{alg:plated-sumproduct} must eventually either fail or reach a factor $f_u$ in leaf plate set $L_u\ni a$ where the plate $a$ is \textsc{Product}-reduced.
Similarly at another time, the algorithm must reach a factor $f_w$ in plate set $L_w\ni b$ where the plate $b$ is \textsc{Product}-reduced.

Consider the provenance of factors $f_u$ and $f_w$.
Since $v$ contains a common factor $f_v$ in both plates $a,b$, factors $f_u$ and $f_w$ must share $f_v$ as a common ancestor.
Note that in Algorithm~\ref{alg:plated-sumproduct}'s search for a next plate set $L'\subsetneq L$, the plate set is strictly decreasing.
Thus if $f_u$ and $f_v$ share a common ancestor, then either $L_u\subseteq L_w$ or $L_u\supseteq L_w$.
But this contradicts $a\in L_u\setminus L_w$ and $b\in L_w\setminus L_u$.
\end{proof}

\subsection{Proof of Theorem~\ref{thm:complexity}} \label{sec:proof-complexity}

We first define plated junction trees, then prove a crucial lemma, and finally prove a stronger form of Theorem~\ref{thm:complexity}.

\begin{definition} \label{def:junction-tree}
A \emph{junction tree} (or \emph{tree decomposition}) of a factor graph $(V,F,E)$ is a tree $(V_J,E_J)$ whose vertices are subsets of variables $V$ such that:
(\emph{i}) each variable is in at least one junction node (i.e.~ $\bigcup V_J = V$);
(\emph{ii}) for each factor $f\in F$, there is a junction vertex $v_J\supseteq \{v\in V\mid(v,f)\in E\}$ containing all of that factor's variables; and
(\emph{iii}) each variable $v\in V$ appears in nodes $\{v_J\in V_J\mid v\in v_J\}$ forming a contiguous subtree of $E_J$.
\end{definition}

\begin{definition} \label{def:plated-junction-tree}
A \emph{plated junction tree} of a plated factor graph $(V,F,E,P\!:\!V\cup F\to\mathcal P(B))$ is a plated graph $(V_J,\,E_J,\,P_J\!:\!V_J\to\mathcal P(B))$ where $(V_J,E_J)$ is a junction tree, $P_J(v_J)=\bigcup_{v\in v_J} P(v)$, and (\emph{iv}) each variable $v\in V$ appears in some junction vertex $v_J\ni v$ with exactly the same plates $P(v_J)=P(v)$.
\end{definition}

We extend unrolling from plated factor graphs to plated junction trees in the obvious way.
Note that a plated junction tree may unroll to a junction graph that is not a tree, since unrolling may introduce cycles.

\begin{definition}
The \emph{width} of a junction tree $T=(V_J,E_J)$ is $\operatorname{width}(T)=\max_{v_J\in V_J} |v_J|-1$.
The \emph{treewidth} of a factor graph $G$ is the width of its narrowest junction tree, ${\operatorname{treewidth}(G)=\min_{T\text{ of }G}\operatorname{width}(T)}$.
\end{definition}

\begin{lemma} \label{thm:crux}
Let $G=(V,\,F,\,E,\,P\!:\!V\cup F\to\mathcal P(B))$ be a plated factor graph and $W\in \mathbb N$.
If for every plate size assignment $M:B\to\mathbb N$ there is a junction tree $T_M$ of $\unroll(G,M,B)$ with $\operatorname{width}(T_M)\le W$, then there is a single plated junction tree $T$ of $G$ such that for every size assignment $M:B\to\mathbb N$, $\operatorname{width}(\unroll(T,M,B))\le W$.
\end{lemma}
\begin{proof}
By induction on $|B|$ it suffices to show that, splitting off a single plate $B=\{b\}\cup B'$, if there is a family of plated junction trees $\{T_m\mid m\in\mathbb N\}$ satisfying
\begin{equation}\label{eq:hypothesis}
\begin{aligned}
    \forall m\in\mathbb N,\;
    \exists T_m \text{ of } \unroll(G, M, b),\;
    \forall M'\!:\!B'\to\mathbb N,\;&
    \\
    \operatorname{width}(\unroll(T_m, M',B'))\le W&
\end{aligned}
\end{equation}
then for a single plated junction tree $T$,
\begin{equation}\label{eq:conclusion}
\begin{aligned}
    &\exists T \text{ of } G,\;
    \forall M\!:\!\{b\}\cup B'\to\mathbb N,\;
    \\
    &\hspace{1cm}\operatorname{width}(\unroll(T,M,\{b\}\cup B'))\le W
\end{aligned}
\end{equation}
Thus let us assume the hypothesis (\ref{eq:hypothesis}) and prove (\ref{eq:conclusion}).

Our strategy is to progressively strengthen the hypothesis (\ref{eq:hypothesis}) by forcing symmetry properties of each $T_m$ using Ramsey-theory arguments.
Eventually we will show that for any $m$, there is a plated junction tree $T_m$ satisfying hypothesis (\ref{eq:hypothesis}) that is the unrolled plated junction tree $\unroll(T, M, b)$ for some plated junction tree $T$.
Choosing $m\ge C$, it follows that $T$ satisfies (\ref{eq:conclusion}).

First we force $T_m$ to have nodes that are symmetric along plate $b$.
Split variables $V_1$ into plated ${V_p=\{v\in V_1\mid b\in P(v)\}}$ and unplated ${V_u=\{v\in V_1\mid b\notin P(v)\}}$ sets.
For each $S_u\subseteq V_u$ and $S_p\subseteq V_p$ color all plate indices $i\in{\{1,\dots,m\}}$ by a coloring function
\[
  C(i) = \begin{cases}
    1 &\exists v_J\in V_J,\, v_J\cap(V_u\cup V_p[i])=S_u\cup S_p[i]
    \\ 0 &\text{otherwise}
  \end{cases}
\]
where we use the shorthand $X[i] = \{x[i]\mid x \in X\}$.
By choosing $m'\ge 2m-1$, we can find $T_{m'}$ with a subset of $m$ plate indices all of a single color.
Rename these plate indices to construct a new $T_m$ whose vertices are symmetric in this sense.

Next for each $u,v\in V_p$, color all pairs of plate indices $i< j\in\{1,\dots,m\}$ by a coloring function
\[
  C(i,j) = \begin{cases}
    1 &\exists v_J\in V_J,\, \{u[i],v[j]\}\subseteq v_J
    \\ 0 &\text{otherwise}
  \end{cases}
\]
By Ramsey's theorem \cite{ramsey2009problem} we can choose a sufficiently large $m'\ge R(m,m)$ such that $T_{m'}$ has a subset of $m$ plates such that all pairs have a single color.
Indeed by choosing $m\ge W$, we can force the color to be 0, since the contiguity property (iii) of Definition~\ref{def:junction-tree} would require a single junction tree node to contain at least $W$ vertices, violating our hypothesis.

At this point we have forced $V_J$ to be symmetric in plate indices up to duplicates, but duplicates may have been introduced in selecting out $m$ plate indices from a larger set $m'$ (e.g. upon removing plate index 4, $\{x\}\!-\!\{x,y[4]\}$ becomes $\{x\}\!-\!\{x\}$).
We now show that these duplicates can be removed by merging them into other nodes.

Let $u,w$ be two duplicate nodes in the plated junction tree $T_m$, so that $P_J(u)=P_J(w)$.
Let $u\!-\!v\!-\!\cdots\!-\!w$ be the path connecting $u,w$ (with possibly $v=w$).
By the contiguity property (iii) of Definition~\ref{def:junction-tree}, $v$ must have at least the variables of $u,w$, hence must have at least as deep plate nesting, i.e.~ $P_J(v)\supseteq P_J(u)$.
Replace the edge $u\!-\!v$ with a new edge $u\!-\!w$.
No cycles can have been created, since $v$ is more deeply plated than $w$.
No instances of the forbidden graph minor can have been created, since the new edge $(u,v)$ lies in a single plate.
Indeed hypothesis~\ref{eq:hypothesis} is still preserved, and symmetry of $V_J$ is preserved.
Merge $u$ into $w$.
Again the hypothesis and symmetry are preserved.
Iterating this merging process we can force $T_m$ to have no duplicate nodes.

Now since $V_J$ is symmetric in plate $b$ and $P_J$ is defined in terms of $V_J$%
, also $P_J$ must be symmetric in plate $b$.
We next force $E_J$ to be symmetric along plate $b$.

For each $S_u,S'_u\subseteq V_u$ and $S_p,S'_p\subseteq V_p$, color all plate indices $i\in\{1\,\dots,m\}$ by a coloring function
\[
  C(i) = \begin{cases}
    1 & (S_u\cup S_p[i],\,S'_u\cup S'_p[i])\in E \\
    0 & \text{otherwise}
  \end{cases}
\]
By choosing $m'\ge 2m-1$, we can find a $T_{m'}$ with a subset of $m$ plate indices with symmetric within-plate-index edges.

For each $S_u,S'_u\subseteq V_u$ and $S_p,S'_p\subseteq V_p$, color all pairs of plate indices $i<j\in\{1,\dots,m\}$ by a coloring function
\[
  C(i,j) = \begin{cases}
    1 & (S_u\cup S_p[i],\,S'_u\cup S'_p[j])\in E \\
    0 &\text{otherwise}
  \end{cases}
\]
By Ramsey's theorem we can choose a sufficiently large $m'\ge R(m,m)$ such that $T_{m'}$ has a subset of $m$ plates such that all pairs have a single color.
Indeed we can force the color to be 0, since for $m\ge 3$ any complete bipartite graph would create cycles, violating the tree assumption.
Hence there are no between-plate-index edges.

At this point, we can construct a $T_m=(V_J,E_J,P_J)$ that is symmetric in plate indices and such that $E_J$ never contains edges between nodes with $b$-plated variables with different $b$ indices.
Hence $T_{\max(W,3)}$ can be rolled into a plated junction tree $T$ that always unrolls to a junction tree.
Finally, since $V_J$ never contains $b$-plated variables with more than one $b$ index,
\[
  \operatorname{width}(\unroll(T,M,B))
  \;=\; \operatorname{width}(\unroll(T,1,B))
  \;\le\; W
\]
\end{proof}

We now prove a strengthening of Theorem~\ref{thm:complexity}.

\begin{theorem} \label{thm:tfae}
Let $G=(V,\,F,\, E,\, P\!:\!V\cup F\to\mathcal P(B))$ be a plated factor graph with nontrivial variable domain $\forall v\in V,\, |\dom(v)|\ge 2$, and $M\!:\!B\to\mathbb N$ be plate sizes.
The following are equivalent:
\begin{enumerate}
    \item Algorithm~\ref{alg:plated-sumproduct} succeeds on $G$.
    \item $\textsc{PlatedSumProduct}(G,M)$ can be computed with complexity polynomial in $M$.
    \item The treewidth of $G$'s unrolled factor graph is (asymptotically) independent of plate sizes $M$.
    \item There is a plated junction tree of $G$ that unrolls to a junction tree for all plate sizes $M$.
    \item There is a plated junction tree of $G$ that has no plated graph minor $\left(\{u,v,w\},\, \left\{(u,v),(v,w)\right\},\, P\right)$ where $P(u)=\{a\}$, $P(v)=\{a,b\}$, $P(w)=\{b\}$, and $a\ne b$.
    \item $G$ has no plated graph minor $\left(\{u,v,w\},\, \left\{(u,v),(v,w)\right\},\, P\right)$ where $P(u)=\{a\}$, $P(v)=\{a,b\}$, $P(w)=\{b\}$, $a\ne b$, and $u,w$ both include variables.
\end{enumerate}
\end{theorem}
\begin{proof}
($1 \Rightarrow 2$)
Algorithm~\ref{alg:plated-sumproduct} has complexity polynomial in $M$, since both \textsc{SumProduct} and \textsc{Product} are polynomial in $M$ and the \textbf{while} loop is bounded independently of $M$.

($2 \Rightarrow 3$) Apply \cite{kwisthout2010necessity} to the unrolled factor graph.

($3 \Rightarrow 4$) Apply Lemma~\ref{thm:crux}.

($4 \Rightarrow 5$) Any plated junction tree with plated graph minor in (5) would unroll to a non-tree, since the minor would induce cycles when $M(a)\ge 2$ and $M(b)\ge 2$ (as in Example~\ref{ex:intractable}).
Hence the plated junction tree of (4) must satisfy (5).

($5 \Rightarrow 6$) If $G$ has such a plated graph minor, then any plated junction tree must have the corresponding plated graph minor.

($6 \Rightarrow 1$) Apply Theorem~\ref{thm:success}.
\end{proof}

\begin{proof}[Proof of Theorem~\ref{thm:complexity}]
Apply Theorem~\ref{thm:tfae} ($1\Longleftrightarrow 2$).
\end{proof}

\section{Experimental details} \label{sec:exp-details}

\subsection{Hidden Markov Models with Autoregressive Likelihoods}
The joint probability (for a single sequence) for the \texttt{HMM} model is given by 
\begin{equation}
\label{eqn:hmmjoint}
p(\bx_{1:T}, \by_{1:T}) = \prod_{t=1}^T p(\by_t | \bx_t) p(\bx_t | \bx_{t-1}) 
\end{equation}
where $\bx_{1:T}$ are the discrete latent variables and $\by_{1:T}$ is the sequence of observations. 
For all twelve models the likelihood is given by a Bernoulli distribution factorized over the 88 distinct notes.
The two Factorial HMMs have two discrete latent variables at each timestep: $\bx_t$ and $\bw_t$. They differ in the dependence structure of $\bx_t$ and $\bw_t$. In particular for the \texttt{FHMM} the joint probability is given by 
\begin{equation}
\label{eqn:fhmmjoint}
p(\bx_{1:T}, \by_{1:T}) = \prod_{t=1}^T p(\by_t | \bx_t, \bw_t) p(\bx_t | \bx_{t-1}) p(\bw_t | \bw_{t-1})
\end{equation}
i.e.~the dependency structure of the discrete latent variables factorizes at each timestep, 
while for the \texttt{PFHMM} (i.e.~Partially Factorial HMM) the joint probability is given by 
\begin{equation}
\label{eqn:fhmmjoint}
p(\bx_{1:T}, \by_{1:T}) = \prod_{t=1}^T p(\by_t | \bx_t, \bw_t) p(\bx_t | \bx_{t-1},\bw_{t} ) p(\bw_t | \bw_{t-1})
\end{equation}
All models correspond to tractable plated factor graphs (in the sense used in the main text) and admit efficient maximum likelihood gradient-based training.

The autoregressive models have the same dependency structure as in Eqn.~\ref{eqn:hmmjoint}-\ref{eqn:fhmmjoint}, with the difference that the likelihood term at each timestep has an additional dependence on $\by_{t-1}$. For the four \texttt{arXXX} models, this dependence is explicitly parameterized with a conditional probability table, with the likelihood for each note $p(y_{t,i}|\cdot)$ depending explicitly on $y_{t-1,i}$ (but not on $y_{t-1,j}$ for $j\ne i$). For the four \texttt{nnXXX} models this dependence is parameterized by a neural network so that $p(y_{t,i}|\cdot)$ depends on the entire vector of notes $\by_{t-1}$. In detail the computational path of the neural network is as follows. First, a 1-dimensional convolution with a kernel size of three and $N_{\rm channels}$ channels is applied to $\by_{t-1}$ to produce $\bm{c}_t$. We then apply separate affine transformations to the latent variables (either $\bx_t$ or $\bx_t$ and $\bw_t$) and ${\bf c}_t$, add together the resulting hidden representations, and apply a ReLU non-linearity. A final affine transformation then maps the hidden units to the 88-dimensional logits space of the Bernoulli likelihood.
We vary $N_{\rm channels} \in \{2, 4\}$ and fix the number of hidden units to 50.

We evaluate our models on three of the polyphonic music datasets considered in \citet{boulanger2012modeling}, using the same train/test splits. Each dataset contains at least 7 hours of polyphonic
music; after pre-processing each dataset consists of $\mathcal{O}(100-1000)$ sequences, with each sequence containing $\mathcal{O}(100-1000)$ timesteps. For each model we do a grid search over hyperparameters and train the model to approximate convergence and report test log likelihoods on the held-out test set (normalized per timestep). For all models except for the second-order HMMs we vary the hidden dimension $D_h \in \{9, 16, 25, 36\}$. For the Factorial HMMs $D_h$ is interpreted as the size of the entire latent space at each timestep so that the dimension of each of the two discrete latent variables $\bx_t$ and $\bw_t$ at each timestep is given by $\sqrt{D_h}$. For the second-order HMMs we vary the hidden dimension $D_h \in \{9, 12, 16, 20\}$ so as to limit total memory usage (which scales as $\mathcal{O}(D_h^\ell)$ for an $\ell^{\rm th}$-order HMM).

We use the Adam optimizer with an initial learning rate of $0.03$ and default momentum hyperparameters \cite{kingma2014adam}; over the course of training we decay the learning rate geometrically to a final learning rate of $3\times 10^{-5}$. For the JSB and Piano datasets we train for up to 300 epochs, while for the Nottingham dataset we train for up to 200 epochs. We follow (noisy) gradient estimates of the log likelihood by subsampling the data into mini-batches of sequences; we use mini-batch sizes of 20, 15, and 30 for the JSB, Piano, and Nottingham datasets, respectively. We clamp all probabilities to satisfy $p \ge p_{\rm min}=10^{-12}$ to avoid numerical instabilities. In order to better avoid bad local minima, for each hyperparameter setting we train with 4 (2) different random number seeds for  models without (with) neural networks in the likelihood, respectively. For each dataset we then report results for the model with the best training log likelihood.

\begin{figure}[ht!]
\begin{center}
\resizebox {.65\columnwidth} {!} {
\begin{tikzpicture}
\definecolor{mygrey}{RGB}{220,220,220}

\node[latent, fill=mygrey](Y1) {$y_{t,i}$};
\node[latent, fill=mygrey, right=of Y1](Y2) {$y_{t+1,i}$};
\node[latent, above=of Y1](Z1) {$x_{t}$};
\node[latent, above=of Z1,xshift=-6ex,yshift=1ex](W1) {$w_{t}$};
\node[latent, above=of Y2, right=of Z1](Z2) {$x_{t+1}$};
\node[latent, above=of Z2,xshift=-6ex,](W2) {$w_{t+1}$};
\node[const, left=of Z1](Z0) {$\cdots$};
\node[const, right=of Z2](Z3) {$\cdots$};
\node[const, left=of W1](W0) {$\cdots$};
\node[const, right=of W2](W3) {$\cdots$};

\edge {W1} {Z1,W2,Y1}
\edge {W2} {Z2,Y2}
\edge {Z0} {Z1}
\edge {W0} {W1}
\edge {W2} {W3}
\edge {Z1} {Z2,Y1}
\edge {Z2} {Y2,Z3}

\plate [inner xsep=0.5cm, inner ysep=0.2cm, yshift=0.0cm] {a}{(Z0)(Z3)(Z1)(Y1)(Z2)(Y2)(W1)(W2)}{}; 
\node [left plate caption=a-wrap]{$N$};

\plate[inner xsep=0.4cm, inner ysep=0.2cm, yshift=0.1cm]{a}{(Y1)}{};
\node [left plate caption=a-wrap]{$88$};

\plate[inner xsep=0.4cm, inner ysep=0.2cm, yshift=0.1cm]{a}{(Y2)}{};
\node [left plate caption=a-wrap]{$88$};

\end{tikzpicture}
}
\end{center}
\caption{
Plate diagram for the \texttt{PFHMM} in Sec.~\ref{sec:hmmexp}. The outermost plate encodes the independence among the $N$ time series, 
while the plates at each time step encode the fact that the likelihood term $p(y_t | \cdot)$ decomposes into a product of $88$ Bernoulli likelihood factors, one for each note $y_{t,i}$.
}
\label{fig:hmm-model}
\end{figure}
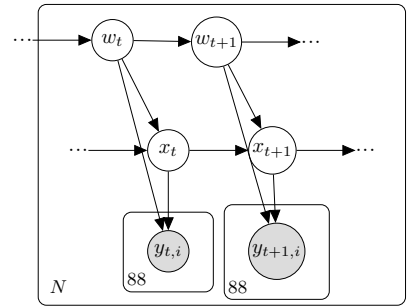

\subsection{Hierarchical Mixed Effect Hidden Markov Models}


\subsubsection{Harbour seal tracking dataset details}
We downloaded the data from \texttt{momentuHMM} (\cite{mcclintock2018momentuhmm}), an R package for analyzing animal movement data with generalized hidden Markov models. 
The raw datapoints are in the form of irregularly sampled time series (datapoints separated by 5-15 minutes on average) of GPS coordinates and diving activity
for each individual in the colony (10 males and 7 females) over the course of a single day
recorded by lightweight tracking devices physically attached to each animal by researchers.
We used the \texttt{momentuHMM} harbour seal example\footnote{\texttt{https://git.io/fjc8i}} preprocessing code
(whose functionality is described in detail in section 3.7 of \cite{mcclintock2018momentuhmm})
to independently convert the raw data for each individual into smoothed, temporally regular time series of step sizes, turn angles, and diving activity, saving the results and using them for our population-level analysis.

\subsubsection{Model details}

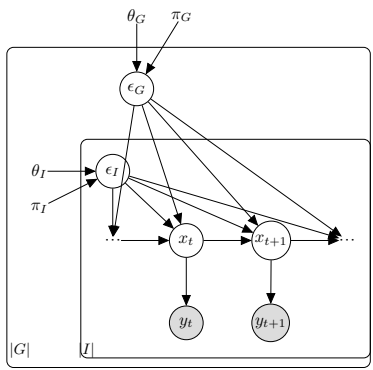
\begin{figure}[ht!]
\begin{center}
\resizebox {.6\columnwidth} {!} {
\begin{tikzpicture}
\definecolor{mygrey}{RGB}{220,220,220}

\node[latent, fill=mygrey](Y1) {$y_t$};
\node[latent, fill=mygrey, right=of Y1](Y2) {$y_{t+1}$};
\node[latent, above=of Y1](Z1) {$x_{t}$};
\node[latent, above=of Y2, right=of Z1](Z2) {$x_{t+1}$};
\node[const, left=of Z1](Z0) {$\cdots$};
\node[const, right=of Z2](Z3) {$\cdots$};

\edge {Z0} {Z1}
\edge {Z1} {Z2,Y1}
\edge {Z2} {Y2,Z3}

\node[latent, above=of Z0](EI) {$\epsilon_I$};
\node[const, left=of EI](TI) {$\theta_I$};
\node[const, below=0.5cm of TI](PI) {$\pi_I$};

\edge {EI} {Z0,Z1,Z2,Z3};
\edge {TI,PI} {EI}

\node[latent, left=of EI, above=of EI, xshift=0.5cm](EG) {$\epsilon_G$};
\node[const, above=of EG](TG) {$\theta_G$};
\node[const, right=0.5cm of TG](PG) {$\pi_G$};

\edge {EG} {Z0,Z1,Z2,Z3};
\edge {TG,PG} {EG}

%

\plate [inner xsep=0.3cm, inner ysep=0.2cm, yshift=0.1cm] {a}{(Z0)(Z3)(Z1)(Y1)(Z2)(Y2)(EI)}{}; 
\node [left plate caption=a-wrap]{$|I|$};

\plate [inner sep=0.4cm, yshift=0.1cm, xshift=-0.1cm] {b}{(Z0)(Z3)(Z1)(Y1)(Z2)(Y2)(EG)(TI)}{}; 
\node [left plate caption=b-wrap] {$|G|$};
\end{tikzpicture}
}
\end{center}
\caption{
A single state transition in the hierarchical mixed-effect hidden Markov model used in our experiments in Section \ref{sec:mehmmexp} and described below. $\theta$s and $\pi$s are learnable parameters. We omit fixed effects from the diagram as there were none in our experiments.
}
\label{fig:mixed-hmm-model}
\end{figure}

Our models are special cases of a time-inhomogeneous discrete state space model
whose state transition distribution is specified by a hierarchical generalized linear mixed model (GLMM).
At each timestep $t$, for each individual trajectory $b \in I$ in each group $a \in G$ (male and female in our experiments), we have
\begin{multline*}
\text{logit}(p(x^{(t)}_{ab} = \text{state } i \mid x^{(t-1)}_{ab} = \text{state } j)) = \\
\big( \epsilon_{G,a} + \epsilon_{I,ab} + Z^\intercal_{I,ab} \beta_{1} + Z^\intercal_{G,a} \beta_{2} + Z^\intercal_{T,abt} \beta_{3} \big)_{ij}
\end{multline*}
where $a,b$ correspond to plate indices, $\epsilon$s are independent random variables, $Z$s are covariates, and $\beta$s are parameter vectors. 
See Fig.~\ref{fig:mixed-hmm-model} for the corresponding plate diagram.
The models in our experiments did not include fixed effects as there were no covariates $Z$ in the harbour seal data, but they are frequently used in the literature with similar models (see e.g. \cite{towner2016sex}) so we include them in our general definition.

The values of the independent random variable $\epsilon_I$ and $\epsilon_G$ are each sampled from a set of three possible values shared across the individual and group plates, respectively.
That is, for each individual trajectory $b \in I$ in each group $a \in G$, we sample single random effect values for an entire trajectory:
\begin{align*}
\iota_{G,a} &\sim \text{Categorical}(\pi_G) \\
\epsilon_{G,a} &= \theta_{G}[\iota_{G,a}] \\
\iota_{I,ab} &\sim \text{Categorical}(\pi_{I,a}) \\
\epsilon_{I,ab} &= \theta_{I,a}[\iota_{I,ab}]
\end{align*}
Note that each $\epsilon$ is a $D_h \times D_h$ matrix, where $D_h=3$ is the number of hidden states per timestep in the HMM.

Observations $y^{(t)}$ are represented as sequences of real-valued step lengths and turn angles, modelled by zero-inflated Gamma and von Mises likelihoods respectively.
The seal models also include a third observed variable indicating the amount of diving activity between successive locations, which we model with a zero-inflated Beta distribution following \cite{mcclintock2018momentuhmm}.
We grouped animals by sex and implemented versions of this model with (i) no random effects (as a baseline), and with random effects present at the (ii) group, (iii) individual, or (iv) group+individual levels.

We chose the Gamma and von Mises likelihoods because they were the ones most frequently used in other papers describing
the application of similar models to animal movement data, e.g. \cite{towner2016sex};
another combination, used in \cite{mcclintock2018momentuhmm}, modelled step length with a Weibull distribution
and turn angle with a wrapped Cauchy distribution.

Unlike our models, the models in \cite{mcclintock2013combining,mcclintock2018momentuhmm} incorporate substantial additional prior knowledge in the form of hard constraints on various parameters and random variables (e.g. a maximum step length corresponding to the maximum distance a harbour seal can swim in the interval of a single timestep).

\subsubsection{Training details}

We used the Adam optimizer with initial learning rate $0.05$ and default momentum hyperparameters \cite{kingma2014adam}, annealing the learning rate geometrically by a factor of $0.1$ when the training loss stopped decreasing.
We trained the models for $300$ epochs with $5$ restarts from random initializations, using batch gradient descent because the number of individuals ($17$) was relatively small.
The number of random effect parameter values was taken from \cite{mcclintock2018momentuhmm} and all other hyperparameters were set by choosing the best values on the model with no random effects.

\subsection{Sentiment Analysis}

\subsubsection{Dataset details}
Sentences in Sentihood were collected from Yahoo! Answers by filtering for answers about neighbourhoods of London. Specific neighbourhood mentions were replaced with generic \texttt{location1} or \texttt{location2} tokens.
We follow the previous work \cite{sentihood,sentic,rentnet} and restrict training and evaluation to the 4 most common aspects: price, safety, transit-location, and general.

To give a clearer picture of the task, consider the sentence ``\texttt{Location1} is more expensive but has a better quality of life than \texttt{Location2}'', the labels encode that with respect to \texttt{Location1} the sentence is \textit{negative} in aspect \textit{price}, but \textit{positive} in \textit{general}. 
Similarly \texttt{Location2} would have the opposite sentiments in those two aspects. The remaining aspects for both locations would have the \textit{none} sentiment.

We preprocess the text with SpaCy \cite{spacy} and lowercase all words. We append the reserved symbols `$\langle bos \rangle$' and `$\langle eos \rangle$' to the start and end of all sentences.

\subsubsection{Model details}
Our reimplemented \texttt{LSTM-Final} baseline uses a BLSTM on top of word embeddings. The hidden state of the BLSTM is initialized with an embedding of the location and aspect. A projection of the final hidden state is then used to classify the sentence-level sentiment by applying the softmax transformation.
For the \texttt{CRF-LSTM-Diag} model, the potentials of the graphical model are given by:
\begin{equation}
\begin{aligned}
G_t(z_t, l, a, \mathbf{x}) &= W[z_t]^T\textrm{LSTM}(\textrm{Emb}(\mathbf{x}), \textrm{Emb}(l, a))[t])\\
F_t(y, z_t, l, a, \mathbf{x}) &= \textrm{diag}(\theta_{\textrm{none}},\theta_{\textrm{pos}},\theta_{\textrm{neg}})[y,z_t]\\
\end{aligned}
\label{eqn:crf}
\end{equation}
where $W\in\mathbb{R}^{3\times D}$ contains a $D$ dimensional vector for each sentiment, $D$ is the dimensionality of the LSTM, and the output of the LSTM is of dimension $T\times D$. The LSTM function takes as input a sequence of word embeddings as well as an initial hidden state given by embedding the location and aspect. The Emb function projects the words or location and aspect into a low-dimensional representation.
The potentials of the \texttt{CRF-LSTM-LSTM} model are given by:
\begin{equation}
\begin{aligned}
G_t(z_t, a, l,\mathbf{x}) &= W_{a,l}[z_t]^T\textrm{LSTM}(\textrm{Emb}(\mathbf{x}),\textrm{Emb}(l, a))[t])\\
F_t(y, z_t, a, l,\mathbf{x}) &= \begin{pmatrix}
\begin{matrix}\theta_\textrm{none} \end{matrix} & \begin{matrix}0 & 0\end{matrix}\\
\begin{matrix}0\\0 \end{matrix} & 
\Scale[1.3]{M^t}\\
\end{pmatrix}[y,z_t]\\
\end{aligned}
\label{eqn:crf}
\end{equation}
where
$M^t = W_M^T\textrm{LSTM}(\textrm{Emb}(\mathbf{x}),\textrm{Emb}(a,l))[t]$ is a matrix that dictates the interaction between positive/negative word sentiment and positive/negative sentence sentiment. The projection $W_M\in\mathbb{R}^{D\times 2^2}$ where $D$ is the dimensionality of the LSTM's output. $M^t$ is then reshaped into an $2\times2$ matrix. 
The potentials of the \texttt{CRF-Emb-LSTM} model, are similar:
\begin{equation}
\begin{aligned}
G_t(z_t, a, l,\mathbf{x}) &= W_{a,l}[z_t]^T\textrm{Emb}(\mathbf{x})[t])\\
F_t(y, z_t, a, l,\mathbf{x}) &= \begin{pmatrix}
\begin{matrix}\theta_\textrm{none} \end{matrix} & \begin{matrix}0 & 0\end{matrix}\\
\begin{matrix}0\\0 \end{matrix} &
\Scale[1.3]{M^t}\\
\end{pmatrix}[y,z_t]\\
\end{aligned}
\label{eqn:crf}
\end{equation}
where $M^t$ is defined above.

\subsubsection{Training details}
Since we treat each tuple $(\bx,a,l,y)$ as a separate example,
we create a class imbalance problem as most sentiments are \textit{none}.
Thus during training we subsample to ensure that every batch has an equal number of none, positive, and negative sentiments. In each epoch we iterate over all examples from the smallest class, in this case the \textit{negative} class, and subsample the rest. In all experiments we use a batch size of 33 during training. We utilize the Adam optimizer \cite{kingma2014adam} with a learning rate of 0.01 and default momentum parameters. We do not decay the learning rate during training and select the final model based on the validation sentiment accuracy at each epoch. We terminate training after 1000 epochs.

\subsubsection{Model parameters}
For all models we utilize the GloVe 840B 300D embeddings \cite{glove} to initialize
the word embeddings and do not update the word embeddings during training.
Since the dataset is extremely small, we found that the word embeddings had a very large impact on performance.
We also use a 2-layer BLSTM with 50 hidden dimensions.
Dropout is used with probability 0.2 after the embeddings and in the BLSTM.

\subsubsection{Evaluation}
The primary evaluation metrics we use for Sentihood are the sentiment accuracy and the aspect macro-F1 score. We calculate accuracy over examples with non-`none' gold sentiment labels. Similarly, we follow Ma et. al. \cite{sentic} in their calculation of the macro F1 score by predicting the sentiment of all location and aspect pairs for a given sentence and using the number of correctly predicted sentiments among the non-none gold sentiments for the precision and recall. We ignore sentences with no non-none gold sentiments.

\section{Walking through Algorithm~\ref{alg:plated-sumproduct}}

\subsection{Walking through the intractable Example~\ref{ex:intractable}}
\label{sec:walkthrough-1}

Consider the intractable model Example~\ref{ex:intractable}, which is the minimal plated factor graph for which Algorithm~\ref{alg:plated-sumproduct} fails:
\begin{align*}
    V &= \{x,y\}\\
    F &= \{f_{xy}\}\\
    E &= \{(x,f_{xy}), (y,f_{xy})\} \\
    P &= \{(x, \{a\}), (y,\{b\}), (f_{xy},\{a,b\})\}
\end{align*}
On the first pass through the \textbf{while} loop, there is a single choice of leaf and a single connected component
\begin{align*}
    L &\from\{a,b\} \\
    V_L &\from \{\} &&= V_c \\
    F_L &\from \{f\} &&= F_c \\
    E_L &\from \{\} &&= E_c
\end{align*}
At this point no variable can be eliminated since $V_c=\emptyset$:
\begin{align*}
    f &\from \textsc{SumProduct}(\{f_{xy}\}, \emptyset) &&= f_{xy} \\
    V_f &\from \{x,y\}
\end{align*}
Since $V_f$ is not empty, we search for a next plate set but find
\begin{align*}
    L' \from P(x) \cup P(y) \quad = \{a\} \cup \{b\} \quad = L
\end{align*}
Now since $L'=L$ the algorithm cannot progress and results in \textbf{error}.

\subsection{Walking through the experimental model \ref{sec:performance}}
\label{sec:walkthrough-3}

Consider the experimental model of \ref{sec:performance} with
\begin{align*}
    V &= \{v,w,x,y,z\}\\
    F &= \{f_{vw}, f_{wx}, f_x, f_{xy}, f_{yz}\}\\
    E &= \{(v,f_{vw}), (w, f_{vw}), (w, f_{wx}), (x, f_{wx}), (x, f_x),\\
      &\quad\quad (x, f_{xy}), (y, f_{xy}), (y, f_{yz}), (z, f_{yz})\} \\
    P &= \{(v, \{a,b\}), (w,\{a\}), (x,\{\}), (y,\{b\}), (z, \{a,b\}),
      \\&\quad\quad
      (f_{vw}, \{a,b\}),
      (f_{wx}, \{a\}),
      (f_x, \{\}),
      \\&\quad\quad
      (f_{xy}, \{b\}),
      (f_{yz}, \{a,b\})
      \}
\end{align*}
On the first pass through the \textbf{while} loop, there is a single choice of leaf:
\begin{align*}
    L & =\{a,b\} \\
    V_L &= \{v,z\} \\
    F_L &= \{f_{vw}, f_{yz}\} \\
    E_L & = \{(v,f_{vw}), (z,f_{yz})\}
\end{align*}
There will be two connected components, one with $v$ and one with $z$.
We process them in an arbitrary order.
\begin{enumerate}
    \item Connected component $V_c=\{v\},\,F_c=\{f_{vw}\}$:\\
    We first eliminate the variable $v$ via a sum-reduction and record that variable $w$ remains to be eliminated:
    \begin{align*}
        f &\from \textsc{SumProduct}(\{f_{vw}\}, \{v\}) \\
        V_f &\from \{w\}
    \end{align*}
    In searching for the next plate set, we find ${L' \from \{a\}}$ and product-reduce plate $b$:
    \[
        f'\from \textsc{Product}(f,\{b\},M)
    \]
    After adding the new factor $f'=:\hat f_w$ and updating data structures, we have
    \begin{align*}
        V &= \{w,x,y,z\}\\
        F &= \{\hat f_w, f_{wx}, f_x, f_{xy}, f_{yz}\}\\
        E &= \{(w,\hat f_w), (w, f_{vw}), (w, f_{wx}), (x, f_{wx}), (x, f_x),\\
          &\quad\quad (x, f_{xy}), (y, f_{xy}), (y, f_{yz}), (z, f_{yz})\} \\
        P &= \{(w,\{a\}), (x,\{\}), (y,\{b\}), (z, \{a,b\}),
          \\&\quad\quad
          (\hat f_w, \{a,b\}),
          (f_{wx}, \{a\}),
          (f_x, \{\}),
          \\&\quad\quad
          (f_{xy}, \{b\}),
          (f_{yz}, \{a,b\})
          \}
    \end{align*}
    \item Connected component $V_c=\{z\},\,F_c=\{f_{yz}\}$:\\
    Operating similarly with $z$ we have
    \begin{align*}
        V &= \{w,x,y\}\\
        F &= \{\hat f_w, f_{wx}, f_x, f_{xy}, \hat f_y\}\\
        E &= \{(w,\hat f_w), (w, f_{vw}), (w, f_{wx}), (x, f_{wx}), (x, f_x),\\
          &\quad\quad (x, f_{xy}), (y, f_{xy}), (y, \hat f_y)\} \\
        P &= \{(w,\{a\}), (x,\{\}), (y,\{b\})),
          \\&\quad\quad
          (\hat f_w, \{a,b\}),
          (f_{wx}, \{a\}),
          (f_x, \{\}),
          \\&\quad\quad
          (f_{xy}, \{b\}),
          (\hat f_y, \{b\})
          \}
    \end{align*}
\end{enumerate}

On the second pass through the \textbf{while} loop, there are two possible leaves, $L=\{a\}$ or $L=\{b\}$.
Arbitrarily choosing leaf $a$, we set
\begin{align*}
    L &= \{a\} \\
    V_L &= \{w\} \\
    F_L &= \{\hat f_w, f_{wx}\} \\
    E_L & = \{(w, \hat f_w), (w,f_{wx})\}
\end{align*}
There is a single connected component with ${V_c=V_L}$, ${F_c=F_L}$.
We eliminate variable $w$ via a vector-matrix multiply  and record that variable $x$ remains to be eliminated:
\begin{align*}
    f &\from \textsc{SumProduct}(\{\hat f_w, f_{wx}\}, \{w\}) \\
    V_f &\from \{x\}
\end{align*}
The next plate set is $L'\from\emptyset$, so we product-reduce plate $a$:
\[
  f'\from \textsc{Product}(f,\{a\},M)
\]
After adding the new factor $f'=:\hat f_x$ and updating data structures, we have
\begin{align*}
    V &= \{x,y\}\\
    F &= \{\hat f_x, f_x, f_{xy}, \hat f_y\}\\
    E &= \{(x,\hat f_x), (x, f_x), (x, f_{xy}), (y, f_{xy}), (y, \hat f_y)\} \\
    P &= \{(x,\{\}), (y,\{b\}),
      \\&\quad\quad
      (\hat f_x, \{\}),
      (f_x, \{\}),
      (f_{xy}, \{b\}),
      (\hat f_y, \{b\})
      \}
\end{align*}

On the third pass through the \textbf{while} loop we choose ${L\from\{b\}}$ and similarly eliminate $y$, resulting in
\begin{align*}
    V &= \{x\}\\
    F &= \{\hat f_x, f_x, \hat f_x'\}\\
    E &= \{(x,\hat f_x), (x, f_x), (x, \hat f_x')\} \\
    P &= \{(x,\{\}),
      (\hat f_x, \{\}),
      (f_x, \{\}),
      (\hat f_x', \{\})
      \}
\end{align*}

On the final pass through the \textbf{while} loop, we choose ${L\from \emptyset}$.
We eliminate $x$ via a three-way dot product and record that no more variables remain to eliminate:
\begin{align*}
    f &\from \textsc{SumProduct}(\{\hat f_w, f_x, \hat f_x'\}, \{x\}) \\
    V_f &\from \emptyset
\end{align*}
Since $V_f$ is empty we add the new factor ${\textsc{Product}(f, \emptyset, M)=:\hat f}$ to scalars $S$.
Note that in this example the final \textsc{Product} is a no-op, however it would have been nontrivial if some plate had contained all variables, as e.g. happens when training on a minibatch of data.

Finally we \textbf{return} $\textsc{SumProduct}(\{f\}, \emptyset)$; again this final operation is a no-op since there was only one connected component in the input plated factor graph.

\section{Plates in Pyro}
\label{sec:pyro-plate}

The Pyro probabilistic programming language provides a Python context manager \lstinline$pyro.plate$ to declare that a portion of a probabilistic model is replicated over a tensor dimension and is statistically independent over that dimension.
An example of an independent dimension is the index over data in a minibatch: each datum should be independent of all others.

To declare a component of a model as plated, a user writes sample statements in a context, e.g.
\begin{lstlisting}[language=Python]
with pyro.plate("my_plate", 100, dim=-1):
    x = pyro.sample("x", Bernoulli(0.5))
    y = pyro.sample("y", Bernoulli(0.1 + 0.8 * x))
\end{lstlisting}
In the above example, the distribution at sample site \lstinline$"x"$ is expanded from 1 to 100 conditinally independent samples, and \lstinline$x$ will have shape \lstinline$(100,)$.
The Bernoulli distribution over \lstinline$x$ is already batched because its parameters depend on \lstinline$x$, hence it does not need to be expanded.

Plates can be nested to account for multiple dimensions
\begin{lstlisting}[language=Python]
with pyro.plate("x_axis", 320, dim=-1):
    # within this context, batch dimension -1 is independent
    with pyro.plate("y_axis", 200, dim=-2):
        # within this context, batch dimensions -2 and -1 are independent
\end{lstlisting}
Note that dimensions use negative indices to follow the NumPy convention of counting from the right of a tensor shape; this allows indices to be compatible with tensor broadcasting.

To create overlapping plates with non-strictly-nested relationship, users can create the context managers beforehand and enter the appropriate contexts at each sample statement.
For example in a model where some noise depends on an x position, some noise depends only on y position, and some noise depends on both, users can write:
\begin{lstlisting}[language=Python]
x_axis = pyro.plate("x_axis", 320, dim=-2)
y_axis = pyro.plate("y_axis", 200, dim=-3)
with x_axis:
    # within this context, batch dimension -2 is independent
with y_axis:
    # within this context, batch dimension -3 is independent
with x_axis, y_axis:
    # within this context, batch dimensions -3 and -2 are independent
\end{lstlisting}

\end{document}